\documentclass[review]{elsarticle}

\usepackage[english]{babel}
\usepackage[utf8x]{inputenc}
\usepackage{amsmath}
\usepackage{amsthm}
\usepackage{dsfont}
\newtheorem{problem}{Problem}
\newtheorem{remark}{Remark}
\newtheorem{theorem}{Theorem}
\DeclareMathOperator*{\argmin}{arg\,min} 
\usepackage{graphicx,subfigure}
\usepackage{yaacro}
\usepackage[autolanguage]{numprint}
\let\n=\numprint
\npfourdigitnosep
\usepackage{multirow}
\usepackage{lineno,hyperref}
\modulolinenumbers[5]

\journal{Ad Hoc Networks}

\begin{acgroupdef}[list=acronimos]
    \acdef{DOP}{Dubins Orienteering Problem}
    \acdef{MEP}{Minimal Exposure Path}
    \acdef[variantof=MEP]{MEPs}{Minimal Exposure Paths}
    \acdef{WSN}{Wireless Sensor Network}
    \acdef[variantof=WSN]{WSNs}{Wireless Sensor Networks}
    \acdef{NFE}{Numerical Functional Extreme}
    \acdef{HGA}{Hybrid Genetic Algorithm}
    \acdef{POA}{Physarum Optimization Algorithm}
    \acdef{HJB}{Hamilton-Jacobi-Bellman}
    \acdef{CNPq}{Conselho Nacional de Desenvolvimento Científico e Tecnológico}
    \acdef{CAPES}{Coordenação de Aperfeiçoamento de Pessoal de Nível Superior}
    \acdef{FAPEMIG}{Fundação de Amparo à Pesquisa do Estado de Minas Gerais}
\end{acgroupdef}

\begin{document}

\begin{frontmatter}

\title{A Semi-Lagrangian Approach for the Minimal Exposure Path Problem in Wireless Sensor Networks\tnoteref{mytitlenote}}
\tnotetext[mytitlenote]
{
    This work has been financed by the Coordenação de Aperfeiçoamento de Pessoal de Nível Superior - Brasil (CAPES) - Finance Code 001, Conselho Nacional de Desenvolvimento Científico e Tecnológico - Brasil (CNPq) - grant number 432241/2018-3, and the Fundação de Amparo à Pesquisa do Estado de Minas Gerais (FAPEMIG) - grant number APQ-01515-18.
}

\author[mymainaddress]{Armando Alves Neto\corref{mycorrespondingauthor}}
\cortext[mycorrespondingauthor]{Corresponding author}
\ead{aaneto@cpdee.ufmg.br}

\author[mymainaddress]{Víctor C. da Silva Campos}
\ead{victor@cpdee.ufmg.br}

\author[mysecondaryaddress]{Douglas G. Macharet}
\ead{doug@dcc.ufmg.br}

\address[mymainaddress]{Dept. of Electronics Engineering, Univ. Federal de Minas Gerais, Belo Horizonte, Brazil.}

\address[mysecondaryaddress]{Computer Vision and Robotics Laboratory (VeRLab), Dep. of Computer Science, Univ. Federal de Minas Gerais, Belo Horizonte, Brazil}

\begin{abstract}
    A critical metric of the coverage quality in \ac{WSNs} is the \ac{MEP}, a path through the environment that least exposes an intruder to the sensor detecting nodes.
    Many approaches have been proposed in the last decades to solve this optimization problem, ranging from classic (grid-based and Voronoi-based) planners to genetic meta-heuristics. However, most of them are limited to specific sensing models and obstacle-free spaces. Still, none of them guarantee an optimal solution, and the state-of-the-art is expensive in terms of run-time.
    Therefore, in this paper, we propose a novel method that models the \ac{MEP} as an \emph{Optimal Control} problem and solves it by using a \emph{Semi-Lagrangian} approach. This framework is shown to converge to the optimal \ac{MEP} while also incorporates different homogeneous and heterogeneous sensor models and geometric constraints (obstacles).
    Experiments show that our method dominates the state-of-the-art, improving the results by approximately \n[\%]{10} with a relatively lower execution time.
\end{abstract}


\end{frontmatter}

\linenumbers

\newcommand{\escalar}[1]{\ensuremath{\mathit{#1}}}
\newcommand{\vetor}[1]{\ensuremath{\boldsymbol{#1}}}
\newcommand{\matriz}[1]{\ensuremath{\mathbf{\uppercase{#1}}}}
\newcommand{\conjunto}[1]{\ensuremath{\mathcal{\uppercase{#1}}}}
\newcommand{\distribuicao}[1]{\ensuremath{\mathcal{\uppercase{#1}}}}
\newcommand{\transpose}{\ensuremath{{}^\intercal}}
\newcommand{\espaco}[1]{\ensuremath{\mathds{\MakeUppercase#1}}}

\newcommand{\node}{\vetor{n}}
\newcommand{\sensorpos}{\vetor{s}}
\newcommand{\pos}{\vetor{p}}
\newcommand{\controle}{\vetor{u}}
\newcommand{\pinit}{\pos_{\textrm{init}}}
\newcommand{\pgoal}{\pos_{\textrm{goal}}}

\newcommand{\nnodes}{\escalar{n}}
\newcommand{\radius}{\escalar{r}}

\def\onefig{.65}
\def\threefig{.4}

\newcommand{\sensingModel}[1][\cdot]{\escalar{S}\!\left(#1\right)}
\newcommand{\intensityModel}[1][\cdot]{\escalar{I}\!\left(#1\right)}
\newcommand{\exposure}[1][\cdot]{\escalar{E}\!\left(#1\right)}
\newcommand{\iexposure}[1][\cdot]{\escalar{E}_{#1}}

\newcommand{\Vfunc}[1][\cdot]{\escalar{V}\!\left(#1\right)}
\newcommand{\bVfunc}[1][\cdot]{\overline{\escalar{V}}\!\left(#1\right)}
\newcommand{\ibVfunc}[2][\cdot]{\overline{\escalar{V}}_{#2}\!\left(#1\right)}

\newcommand{\Qfunc}[1][\cdot]{\escalar{Q}\!\left(#1\right)}

\newcommand{\dynaFunc}[1][\cdot]{\escalar{f}\!\left(#1\right)}

\newcommand{\nodeset}{\conjunto{N}}
\newcommand{\inputset}{\conjunto{U}}
\newcommand{\Reais}[1]{\ensuremath{\espaco{R}^{#1}}}

\section{Introduction}
\label{sec:introduction}


%

\ace{WSNs} are commonly used in a vast range of civilian and military applications and have been the focus of many studies. Such networks are constituted of multiple stationary wireless nodes with processor, memory, radio transceiver, power source, and a set of sensors used to collect data from the region where they have been deployed.
In this context, a fundamental aspect is the coverage of the \ac{WSN}, i.e., the monitoring quality of the network considering the dispersion of sensors and their properties. Most of the studies consider a full coverage of the area of interest, where every portion of the environment is within the sensing range of at least one sensor \cite{Mohamed2017Coverage}. However, different sensing models can be found in the literature~\cite{Ye2016Hybrid}, and each one represents a target detection differently.


Detection of mobile targets using \ac{WSNs} is a typical usage example of such systems, and it has been formulated as different problems in the literature, such as trap coverage~\cite{Balister2009TrapCoverage}, barrier coverage~\cite{Chen2013EnergyEfficient}, and minimum exposure path~\cite{Meguerdichian2001Exposure,Veltri2003Minimal}.
In the \ace{MEP} problem, the goal is to determine a path through the sensing field that connects two arbitrary positions and minimizes the likelihood of a target being detected by the network during its movement. This is critical since the exposure can also be used as a quality metric of the \ac{WSN}, and the \ac{MEP} represents the worst-case coverage performance. Moreover, this information allows to enhance the network during the design phase or to optimize and maintain it after deployment.
The \ac{MEP} problem has been the focus of many works in the literature and is generally tackled with grid-based approaches \cite{Meguerdichian2001Exposure}, methods based on the use of Voronoi diagrams \cite{Veltri2003Minimal}, and heuristic solutions such as evolutionary algorithms \cite{Ye2016Hybrid}.

%

\begin{figure}[!t]
    \centering
    \includegraphics[width=\onefig\linewidth]{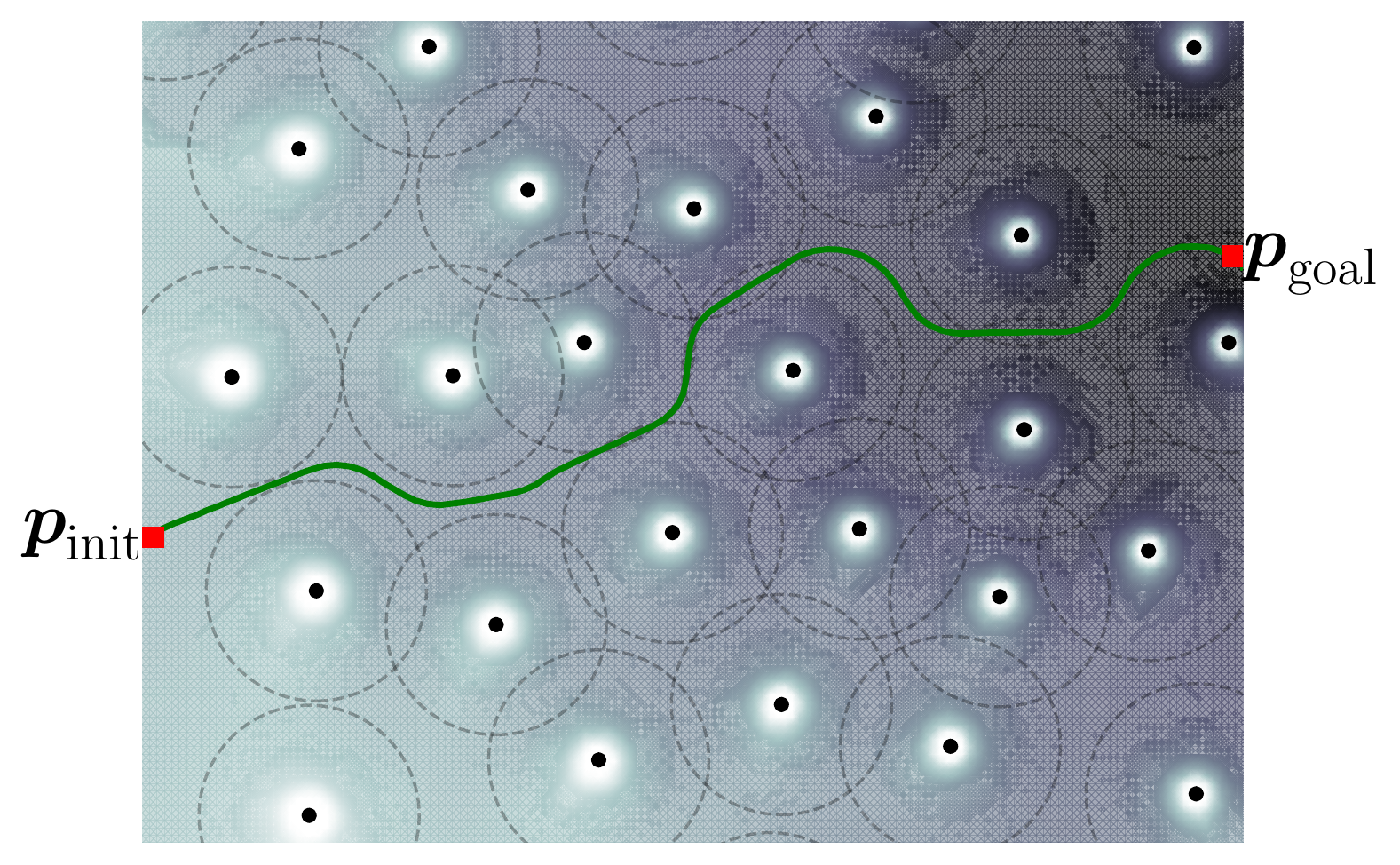}
    \caption{\protect\ac{MEP} problem: find a path between a source ($\pinit$) and destination ($\pgoal$) that minimizes the exposure of a moving target to a set of sensor nodes. Our approach computes a value function representing an \emph{exposure field} over the network, whose gradient approximates the optimal solution.}
    \label{fig:introduction}
\end{figure}

In this paper, we propose a formulation based on optimal control theory for determining optimal control inputs for an intruder in the \ac{MEP} problem. This approach uses the Dynamic Programming Principle to approximate the solution with a Semi-Lagrangian numerical scheme and policy iteration to accelerate the convergence.
Fig.~\ref{fig:introduction} illustrates a scenario filled with sensor nodes, over which we compute a \emph{value function} capable of leading an agent from $\pinit$ to $\pgoal$ while minimizing its exposure to the network topology.
%
%


The methodology was evaluated considering recently available benchmarks, and it overcame the state-of-the-art in all instances. Specifically, we highlight the following contributions:
\begin{itemize}
    \item we proposed a novel approach to solve the \ac{MEP} problem using an optimal control framework that ensures convergence to the optimal solution -- results are on average \n[\%]{10} better than the state-of-the-art.
    
    \item our formulation allows to incorporate the target's dynamics (which produces smoother solutions), as well as many different sensing models and intensity functions for homogeneous and heterogeneous network topologies;
    
    \item it also allows to model geometric constraints, such as obstacles, holes, and other non-navigable areas;
    
    \item the method is faster than the best one in the literature, even when employed for a large number of nodes -- it is even more efficient when computing other paths in the same scenario (for the same $\pgoal$).
\end{itemize}



The remainder of this paper is structured as follows: Sec.~\ref{sec:related_work} reviews existing literature; 
Sec.~\ref{sec:problem_formalization} introduces preliminary concepts and formulates the problem;
Sec.~\ref{sec:methodology} presents and details the proposed approach; 
Sec.~\ref{sec:experiments} shows numerical results; 
and finally, in Sec.~\ref{sec:conclusion}, we conclude and discuss avenues for future investigation.


\section{Related work}
\label{sec:related_work}


Coverage is one of the most fundamental factors regarding \ac{WSNs}, and it is related to problems such as deployment planning and network optimization. The coverage quality can be assessed, for example, by the exposure of an arbitrary path going through the sensing field, where a higher minimum exposure represents better coverage.

A typical scenario for the employment of \ac{WSNs} is related to the detection and tracking of mobile targets. Such a task is important, for example, in wildlife monitoring or, in the case of security applications, to detect possible intruders.
Problems such as trap coverage~\cite{Balister2009TrapCoverage, Chen2013Trapping} and barrier coverage~\cite{Chen2013EnergyEfficient} are usually concerned about the detection of the target that reaches the sensing range of a node, and the path it takes inside the monitored area is generally not considered. 
On the other hand, the \ace{MEP} problem allows for a broader understanding and representation of the strengths and weaknesses of the \ac{WSN} under consideration, as it serves to better characterize the behavior of a target accordingly to its movement inside the field \cite{Meguerdichian2001Exposure, Veltri2003Minimal}. It is usually formulated as a trajectory optimization problem, which considers the path exposure as the cost function to be minimized. 


Analytical solutions to the \ac{MEP} problem only exist in the trivial case where a single sensor node is considered~\cite{Feng2016Novel}. There is no exact solution to the multiple sensor nodes case, and it has not yet been proven that it is solvable under such circumstance \cite{Djidjev2007Efficient}. In this sense, different approximation solutions have been proposed to tackle this scenario \cite{Djidjev2010Approximation}.



Approaches tackling the \ac{MEP} problem are generally separated into two categories, depending on the node sensing function used: i) \emph{idealistic} models; and ii) \emph{realistic probability} models.
Idealistic models are more straightforward representations of the signal attenuation in \ac{WSNs}, and consequently, the \ac{MEP} can be solved by simpler algorithms. The first studies in this context used grid-based approaches and Voronoi diagrams~\cite{Veltri2003Minimal, Phipatanasuphorn2004Vulnerability, Djidjev2007Efficient}. The main problem with such methods is that, although easy to implement, they cannot be applied to scenarios where multiple sensors are used to detect the intruder (called the all-sensor intensity model). Also, when destination and source points do not lie in the Voronoi diagram, such approaches are unable to provide the optimal solution~\cite{Binh2019Efficient}.

Later, the results were improved by using hybrid evolutionary algorithms, and other biological-inspired solutions for fixed and mobile nodes \cite{Ye2016Hybrid, Nguyen2017Genetic}. 
For example, in \cite{Song2014Biology}, a \ac{POA}, that mimics the behavior of such organism, is used. The application of meta-heuristics provides accurate results since they usually do not rely on approximate cell/grid representations.
To improve performance and efficiency, the authors of \cite{Aravinth2021Hybrid} also proposed a \ac{HGA} based on particle swarm optimization. 



More realistic sensor models have been investigated in recent articles, such as directional nodes \cite{Liu2014Minimal} and sensing probability models \cite{Ye2015Hybrid}.
%
%
To tackle the drawbacks of grid-based and Voronoi-based approaches, \cite{Ye2016Hybrid} presents an \ac{HGA} based on a \ac{NFE} model, transforming the \ac{MEP} into a global optimization problem. Results overcame the state-of-the-art at that moment, however, similarly to other existing works, the method was based on a $m$ x $m$ grid, whose resolution influences the output precision.

Currently, the state-of-the-art in the \ac{MEP} problem is \cite{Binh2019Efficient}, which presents two approaches to solve it: (i) the \emph{GB-MEP}, a grid-based method with some modifications; and (ii) the \emph{GA-MEP}, based on a genetic algorithm with a featured individual representation and an effective combination of genetic operators.
Since it is based on a meta-heuristic algorithm, the GA-MEP provides better precision performance than the GB-MEP between \n[\%]{77} to \n[\%]{88} of the cases, depending on the nodes distribution (Gaussian distribution is only \n[\%]{40}). On the other hand, the GB-MEP is much faster, depending on the resolution used to compute $\exposure$.
Also, when compared with the genetic algorithm proposed in \cite{Ye2016Hybrid}, the GA-MEP dominates the \emph{HGA-NFE}, winning in all tested instances.


Other aspects of the environment and the sensor field are also important to be treated in the \ac{MEP} context. For example, in \cite{Ferrari2010Potential}, an artificial-potential approach that considers the presence of static obstacles and evaluates the paths for multiple vehicles is proposed. 
%
Besides also considering cluttered environments, in \cite{Liu2017Obstacle}, an adaptive cell decomposition approach capable of handling heterogeneous \ac{WSNs} is proposed. Obstacles and heterogeneous topologies are typically more realistic and difficult to deal with than other scenarios.

More recently, the \ac{MEP} formulation was combined with the \ac{DOP}, and a multi-objective evolutionary algorithm was proposed \cite{Macharet2021Minimal}. In this context, besides minimizing the exposure, the path is subjected to a bounded curvature and limited budget. Moreover, the goal is to visit a set of points that maximizes the total reward collected.


In this work, the \ac{MEP} problem is considered in an optimal control setting, and a Dynamic Programming Approach is employed to solve it. Considering a continuous-time agent, it amounts to solving a partial differential equation, known as the \acd{HJB} equation, which can be solved numerically, among other possibilities, by Semi-Lagrangian approximation schemes \cite{Falcone2013}. 
To handle possibly infinite values for the exposure, representing unreachable or forbidden regions on the search space (due to the presence of obstacles), we propose the use of a Kruzkov transformation \cite{Bardi1997, Falcone2013, Chilan2019}, which maps values in $[0, \infty)$ to the $[0,1]$ range. A suitable Dynamic Programming Principle is presented for this transformed case, which is employed to propose a Semi-Lagrangian Approximation scheme to solve the optimal control problem by discretizing it in time and employing an unstructured grid in space. Our approach is shown to converge to the optimal solution of the problem utilizing the Barles-Souganidis Theorem \cite[Theorem 2.1]{Barles1991}.


\section{Problem formulation}
\label{sec:problem_formalization}

In the \ac{MEP} problem, the exposure function $\exposure$  presents two main aspects. First, it depends on the model of the sensors composing the network, here referred to as the \emph{sensing function} $\sensingModel$. Second, it depends on how the combined energy of all sensors is computed at each location, defined as the \emph{intensity function} $\intensityModel$. Next, we discuss these aspects and formalize the \acd{MEP} as a problem of nonlinear optimization.

\subsection{Node sensing functions}
\label{subsec:sensing_models}

According to the literature, there are different sensing models used in \ac{WSNs}, but all of them depend on the Euclidean distance between the sensor location $\sensorpos$ and the position $\pos$ for which we must compute the detection. The most commonly used sensor functions are described below.

\begin{itemize}
    \item Boolean disk coverage model \cite{Ye2016Hybrid},
    \begin{equation}
        \sensingModel[\sensorpos, \pos] = 
        \begin{cases}
            1, & \text{if~} \|\sensorpos - \pos\| \leq \radius\\
            0, & \text{otherwise}
        \end{cases},
        \label{eq:boolean_disk_coverage_model}
    \end{equation}
    where $\radius$ is the critical sensing range. As the name suggests, it emulates an \emph{on-off} detection behavior, which might be used in very simple scenarios.
    
    \item Attenuated disk coverage model \cite{Ye2016Hybrid, Binh2019Efficient},
    \begin{equation}
        \sensingModel[\sensorpos, \pos] = \frac{\lambda}{\|\sensorpos - \pos\|^{\mu}},
        \label{eq:attenuated_disk_coverage_model}
    \end{equation}
    where $\lambda$ and $\mu$ are positive parameters related to the propagation and attenuation of the node signal. 
    
    \item Sensing probability model \cite{Ye2016Hybrid},
    \begin{equation}
        \sensingModel[\sensorpos, \pos] = e^{-\alpha \|\sensorpos - \pos\|^\beta},
        \label{eq:probability_coverage_model}
    \end{equation}
    where $\alpha$ and $\beta$ also represent positive parameters of signal attenuation.
    
    \item Probability coverage model with noise \cite{Deng2007Minimum,Binh2019Efficient},
    \begin{equation}
        \sensingModel[\sensorpos, \pos] = \frac{\lambda}{\|\sensorpos - \pos\|^{\mu}} + \eta,
        \label{eq:probability_coverage_model_noise}
    \end{equation}
    where $\eta$ represents an additive noise energy factor, generally modeled as a Gaussian distribution.   
\end{itemize}

Generally speaking, functions \eqref{eq:boolean_disk_coverage_model} and \eqref{eq:attenuated_disk_coverage_model} are known as \emph{idealistic models}, while \eqref{eq:probability_coverage_model} and \eqref{eq:probability_coverage_model_noise} are considered to be \emph{realistic probability} representations of the real-world. 
In \emph{homogeneous} networks, the same function and parameters are used to model all sensors, while in \emph{heterogeneous} topologies, they are distinct.

\subsection{Field intensity functions}
\label{subsec:intensity_function}

A second aspect of the sensor field is the impact of all $\nnodes$ sensors over the exposure function. Below we describe two main models found in the current literature.

\begin{itemize}
    \item All-sensor field intensity function \cite{Binh2019Efficient},
    \begin{equation}
        \intensityModel[\pos] = \sum_{i=1}^{\nnodes} \sensingModel[\sensorpos_i, \pos],
        \label{eq:intensity_allsensors}
    \end{equation}
    where the influence of all sensors in the field is incorporated in the computation of the exposure, no matter how far they are from $\pos$.
    
    \item Maximum-sensor intensity function \cite{Ye2016Hybrid},
    \begin{equation}
       \intensityModel[\pos] = \max_{i} \sensingModel[\sensorpos_{i}, \pos],
       \label{eq:intensity_maximumsensor}
    \end{equation}
    where only the influence of the node with the highest sensing function is used to compute the exposure. When the field is composed of homogeneous nodes, this sensor is the closest to $\pos$.
\end{itemize}

\subsection{Minimal Exposure Path}
\label{subsec:mep_problem}

Now we can compute the exposure $\exposure$ of a path $\pos(t)$, which is directly related to the \ac{WSN} ability to detect mobile targets traversing its sensing field. The exposure is defined as the time integral of the cumulative energy perceived by the sensors in the network \cite{Meguerdichian2001Exposure, Veltri2003Minimal}, according to a given intensity function. More formally:
\begin{equation}
    \exposure[\pos(t)] = \int_{0}^{t_f} \intensityModel[\pos(t)] \left| \dot{\pos}(t) \right| dt,
    \label{eq:exposure}
\end{equation}
\noindent where $t_f$ is the time the goal is reached. Here, $\pinit = \pos(0)$ represents the \emph{source position}, while $\pgoal = \pos(t_f)$ is the \emph{destination position}. In other words, a target moving from $\pinit$ to $\pgoal$ will be exposed to the node set according to \eqref{eq:exposure}.
%
%
%
%
Finally, the \ac{MEP} problem can be defined as finding a path throughout the environment that less exposes the target. Formally:
\begin{problem}[\acd{MEP}]
    Let $\nodeset = \{\node_i\}_{i=1}^\nnodes$ be a sensor field formed by a set of $\nnodes$ nodes distributed in a $\Reais{2}$ (cluttered) environment, each node represented by a sensing model $\sensingModel[\cdot]: \Reais{2} \rightarrow \Reais{}$. 
    In this context, the main goal is to compute a penetration path $\pos^*$ through $\nodeset$, leading from the source position $\pinit$ to the destination $\pgoal$, that minimizes the exposure \eqref{eq:exposure} of the target, i.e., 
    \begin{equation*}
        \pos^*(t) = \argmin_{\pos(t)} \exposure[\pos(t)].
    \end{equation*}
\end{problem}

\section{Methodology}
\label{sec:methodology}

\subsection{Numerical approximation scheme}

In this paper, we tackle the \ac{MEP} problem in an optimal control setting. First, let us consider a path generated by the dynamic system:
\begin{align}
    \dot{\pos}(t) = \controle(t), \text{~~with~~} \pos(0) = \pinit,
    \label{eq:dyn_sys}
\end{align}
\noindent where $\pos(t)$ is the target position and $\controle(t)$ the velocity vector over time. We consider that the set of admissible control inputs (velocities) is given by:
\begin{equation*}
    \inputset = \left\{\controle \in \Reais{m} ~ | ~ 0 < \controle_{\min} \leq \|\controle(t)\| \leq \controle_{\max} ~ \forall t \right\}.
\end{equation*}
With these definitions, exposure \eqref{eq:exposure} can be rewritten as:
\begin{equation*}
    \exposure[\pos(t),\controle(t)] = \int_0^{t_f} \intensityModel[\pos(t)] \|\controle(t) \| dt.
\end{equation*}
%

Now, we can also define the \emph{value function} $\Vfunc: \Reais{2} \rightarrow \Reais{}$ such that, for every point $\pos(t)$ in the space, determines the minimum exposure from $\pos(t)$ to $\pgoal$:
\begin{align}
    \Vfunc[\pos(t)] = \inf_{\controle \in \inputset} \exposure[\pos(t),\controle(t)]. \label{eq:value_original}
\end{align}
%
%
%
This value function admits the Dynamic Programming Principle (Optimality Principle) since any point along the path should be optimal. Then, it can be written as:
\begin{align}
    \hspace{-2mm}
    \Vfunc[\pos(t)] = \inf_{\controle \in \inputset} \!\left(\Vfunc[\pos_u(\Delta t)] +\! \int_0^{\Delta t} \!\intensityModel[\pos(t)] \|\controle(t) \| dt \right),
    \hspace{-1mm}
    \label{eq:value_function}
\end{align}
with $\pos_u(\Delta t)$ representing the point at time $\Delta t$ along the path, taken when considering velocities defined by $\controle(t)$. In the presented form, the value function can be approximated by employing a Semi-Lagrangian approach \cite{Falcone2013}.  

To deal with constraints along the path, such as obstacles and restricted zones, we consider that the value function must be infinite in those locations. In addition, assuming that the exposure is always non-negative, the value of the goal location must always be null. When considered together, these constraints lead to the boundary conditions:
\begin{align}
    \Vfunc[\pos(t)] = \left\{
    \begin{array}{ll} 
        0, & \textrm{for } \pos(t) = \pgoal \\ 
        \infty, & \textrm{for } \pos(t) \in \partial \mathcal{O} 
    \end{array}\right.,
    \label{eq:boundary_conditions}
\end{align}
with $\partial \mathcal{O}$ representing the boundaries of geometric forbidden regions. To better tackle these infinite values, we employ a rescaling of the value function (known as Kruzkov transformation \cite{Falcone2013}), such that:
\begin{equation}
    \bVfunc[\pos(t)] = 1 - e^{-\Vfunc[\pos(t)]}. 
    \label{eq:value_kruzkov}
\end{equation}
By making use of the original Dynamic Programming Principle of the value function, and concerning:
\begin{equation*}
    \iexposure[\Delta t,\controle] = \int_0^{\Delta t} \intensityModel[\pos(t)] \|\controle(t) \| dt
\end{equation*}
in this transformed setting, we have:
\begin{align}
    \bVfunc[\pos(t)] &= 1 \! - \! e^{-\inf\limits_{\controle \in \inputset} \left(V(\pos_u(\Delta t)) + \iexposure[\Delta t,\controle]\right)}, \nonumber\\
    &= 1 \! - \! \sup_{\controle \in \inputset}\left(e^{-V(\pos_u(\Delta t))}e^{-\iexposure[\Delta t,\controle]}\right), \nonumber\\
    &= 1 \! - \! \sup_{\controle \in \inputset} \! \left(\left(\!-1 \! + \! e^{-V(\pos_u(\Delta t))} \right) e^{-\iexposure[\Delta t,\controle]} \! + \! e^{-\iexposure[\Delta t,\controle]}\right), \nonumber\\
    \bVfunc[\pos(t)] &= 1 \! + \! \inf_{\controle \in \inputset} \! \left(\overline{V}(\pos_u(\Delta t)) \! - \! 1\right) e^{-\int\limits_0^{\Delta t} \intensityModel[\pos(t)] \|\controle(t) \| dt}. \hspace{-2mm}
    \label{eq:dpp_kruzkov}
\end{align}

Eq.~\eqref{eq:dpp_kruzkov} can be seen as the Dynamic Programming Principle for this transformed $\bVfunc$, and we can also employ a Semi-Lagrangian numerical scheme to approximate the solution, by employing a time discretization, followed by a space discretization. 


We consider the time discretization of \eqref{eq:dpp_kruzkov} by applying a trapezoidal approximation for the integral term, and a trapezoidal method to solve the system of equations composed of \eqref{eq:dyn_sys}. By considering a time step of $\Delta t$, this leads to:
\begin{equation}
    \ibVfunc[\pos_k]{k} = 1 \! + \! \inf_{\controle_k \in \inputset}\left(\ibVfunc[\pos_{k+1}]{k+1} - 1\right)e^{-g_k(\pos_k,\controle_k)}, 
    \label{eq:dpp_discrete}
\end{equation}
\noindent with
\begin{equation}
    g_k(\pos_k,\controle_k) = \tfrac{1}{2} \Big( \intensityModel[\pos_k] + \intensityModel[\pos_{k+1}] \Big) \|\controle_k\|\Delta t, 
    \label{eq:gk}
\end{equation}
and 
\begin{equation}
    \pos_{k+1} = \pos_{k} + \controle_k \Delta t. 
    \label{eq:time_step}
\end{equation}


Afterwards, we perform the space discretization of $\bVfunc$, by considering an unstructured grid of points covering the space. Since these points will be used to represent $\bVfunc$, they are the only points over the space for which the value is updated. Since $\ibVfunc[\pos_{k+1}]{k+1}$ might not be a part of the grid, it is replaced by a finite element linear interpolation over the grid. 
In this work, we have employed the Delaunay triangulation on the unstructured grid points to find a triangulation of the space (see Fig.\ref{fig:grid_points}). We consider these triangles as our finite elements, and represent the interpolation of $\ibVfunc[\pos_{k+1}]{k+1}$ as $\mathcal{I}_{\overline{V}_{k+1}}[\pos_{k+1}]$. 
Taken together, both discretizations (time and space) lead to a Semi-Lagrangian approximation scheme of \eqref{eq:dpp_kruzkov}, in the form:
\begin{equation}
    \ibVfunc[\pos_k]{k} = 1 + \inf_{\controle_k \in \inputset}\left(\mathcal{I}_{\overline{V}_{k+1}}[\pos_{k+1}] - 1\right)e^{-g_k(\pos_k,\controle_k)}, 
    \label{eq:dpp_SL}
\end{equation}
with boundary conditions
\begin{equation}
    \ibVfunc[\pos_k]{k} = 
        \left\{\begin{array}{ll}
            0, &\textrm{if } \pos_k = \pgoal \\ 
            1, &\textrm{if } \pos_k \in \partial \mathcal{O} 
        \end{array}\right.. 
    \label{eq:boundary}
\end{equation}

Eq.~\eqref{eq:dpp_SL} can be directly employed in a backwards in time-marching scheme, known as \emph{value iteration}, to find an approximate solution to \eqref{eq:dpp_kruzkov}. Since, in the way the problem has been presented, we are interested in stationary/infinite horizon solutions, we have employed an acceleration technique known as \emph{policy iteration} \cite[Section 8.4.7]{Falcone2013}. In this technique, we alternate between finding an optimal policy $\controle_k$ and an optimal value $\bVfunc$. 
At every grid point, the optimal policy is:
\begin{align}
    \controle_k = \argmin_{\controle_k \in \inputset} \left( \mathcal{I}_{\overline{V}}[\pos_{k+1}] - 1\right)e^{-g_k(\pos_k,\controle_k)}, \label{eq:opt_policy}
\end{align}
and fixed for this iteration. Afterwards, the value function is updated according to \eqref{eq:boundary} and
\begin{equation*}
    \bVfunc[\pos_k] - \mathcal{I}_{\overline{V}}[\pos_{k+1}]e^{-g_k(\pos_k,\controle_k)} = 1  - e^{-g_k(\pos_k,\controle_k)}. 
\end{equation*}
These steps are repeated until the algorithm converges to the minimum of the value function. 
The \acd{MEP} is then found by integrating the system trajectory \eqref{eq:time_step}, starting at $\pos_0 = \pinit$ and employing the optimal policy given by \eqref{eq:opt_policy}, until the path reaches $\pgoal$. 
The original value function in \eqref{eq:value_original} can be recovered by:
\begin{equation*}
    \Vfunc[\pos] = -\ln \left(1 - \bVfunc[\pos] \right).
\end{equation*}

\subsection{Convergence analysis}

The Dynamic Programming Principle of the transformed value function in \eqref{eq:dpp_kruzkov} can be recast as a partial differential equation, known as the \ac{HJB} equation, of the form:
\begin{equation}
    \sup_{\controle \in \inputset} \Big( \bVfunc[\pos] \ell(\pos,\controle) - \nabla \bVfunc[\pos] \cdot \dynaFunc[\pos,\controle] - \ell(\pos,\controle) \Big) = 0 \label{eq:HJB_kruzkov}
\end{equation}
with $\ell(\pos,\controle)$ being the cost, and $\dynaFunc[\pos,\controle]$ defining the system dynamics, given by:
\begin{align}
    \ell(\pos,\controle) &= \intensityModel[\pos] \|\controle\|, \label{eq:ell} \\
    \dynaFunc[\pos,\controle] &= \controle. \nonumber
\end{align}

Having defined this partial differential representation, we are now ready to state the main result concerning the convergence of our proposed methodology.

\begin{theorem}
    Let our optimal control problem be represented by the \ac{HJB} equation \eqref{eq:HJB_kruzkov}. As long as $\ell(\pos,\controle)$ and $\dynaFunc[\pos,\controle]$ are Lipschitz continuous in $\pos$, and $\ell(\pos,\controle) > 0 ~~ \forall \pos, \controle$, there is a unique viscosity solution to \eqref{eq:HJB_kruzkov}, corresponding to the optimal solution to the \ac{MEP} problem. In addition, the proposed numerical solution scheme converges to this unique viscosity solution as the time step, $\Delta t$, and the maximum distance between points on the grid, $\Delta \pos$, tend to zero, so long as $\Delta \pos$ tends faster than $\Delta t$.
\end{theorem}

\begin{proof}
Note that $\ell(\pos,\controle)$ and $\dynaFunc[\pos,\controle]$ being Lipschitz continuous and $\ell(\pos,\controle) > 0 ~~ \forall \pos, \controle$, are sufficient conditions to ensure that there is a unique viscosity solution to \eqref{eq:HJB_kruzkov}, and that the problem admits a comparison principle \cite{Bardi1997}. In this case, from the Barles-Souganidis Theorem \cite[Theorem 2.1]{Barles1991}, as long as our numerical approximation scheme is \emph{monotone}, a \emph{contraction mapping}, and \emph{consistent}, it converges to the unique viscosity solution of the \ac{HJB} equation \cite{Falcone2013}. As such, the remainder of this proof shows these three properties of the approximation scheme in \eqref{eq:dpp_SL}. 

\subsubsection{Monotonicity}

Consider two functions $\overline{W}$ and $\overline{V}$, with $\overline{W} \leq \overline{V}$ for every point on the grid. Suppose that the $\inf$ operator in \eqref{eq:dpp_SL} is attained by $\overline{\mathbf{w}}$ for $\overline{W}$, and $\overline{\controle}$ for $\overline{V}$. It follows that:
\begin{align}
    \overline{W}_{k}(\pos_k) &\leq 1 + \left(\mathcal{I}_{\overline{W}}[\pos_{k+1}] - 1\right)e^{-g_k(\pos_k,\overline{\controle})}, \nonumber \\
    \ibVfunc[\pos_k]{k} - \overline{W}_{k}(\pos_k) &\geq e^{-g_k(\pos_k,\overline{\controle})} \left(\mathcal{I}_{\overline{V}}[\pos_{k+1}] - \mathcal{I}_{\overline{W}}[\pos_{k+1}]\right), \nonumber
\end{align}
which implies that $\ibVfunc[\pos_k]{k} - \overline{W}_{k}(\pos_k) \geq 0$ since we employed a linear interpolation.

\subsubsection{Contractiveness} 

Considering two functions $\overline{W}$ and $\overline{V}$, with $\overline{\mathbf{w}}$ minimizing $\overline{W}$. It follows that:
\begin{align}
    \ibVfunc[\pos_k]{k} - \overline{W}_{k}(\pos_k) &\leq e^{-g_k(\pos_k,\overline{\mathbf{w}})} \left(\mathcal{I}_{\overline{V}}[\pos_{k+1}] - \mathcal{I}_{\overline{W}}[\pos_{k+1}]\right), \nonumber\\
    \ibVfunc[\pos_k]{k} - \overline{W}_{k}(\pos_k) &\leq e^{-g_k(\pos_k,\overline{\mathbf{w}})} \|\overline{V} - \overline{W}\|_\infty^{(k+1)}, \nonumber
\end{align}
with $\|\overline{V} - \overline{W}\|_\infty^{(k+1)}$ being the maximum of the error between the two functions in the next time step. Since we are assuming that $\ell(\pos,\controle) > 0$, then $g_k(\pos_k,\controle_k) > \overline{g}\Delta t > 0 ~~ \forall \pos_k, \controle_k$. Since a similar bound can be found for $\overline{W}_{k}(\pos_k) - \ibVfunc[\pos_k]{k}$, then:
\begin{align}
    \|\overline{V} - \overline{W}\|_\infty^{(k)} \leq e^{-\overline{g}\Delta t} \|\overline{V} - \overline{W}\|_\infty^{(k+1)}. \nonumber
\end{align}
As we solve the problem back in time, this shows that our approximation scheme is a contraction mapping, with contraction rate $e^{-\overline{g}\Delta t}$. From the Banach fixed-point Theorem, it guarantees that our approximation scheme converges to a unique solution.

\subsubsection{Consistency}

We start our consistency analysis by considering the error of time discretization, comparing solutions from \eqref{eq:dpp_kruzkov} and \eqref{eq:dpp_discrete}. If we consider that the $\inf$ operator is attained by $\overline{\controle}$ in \eqref{eq:dpp_discrete}, it follows that:
\begin{align}
    \bVfunc[\pos(t)] - \ibVfunc[\pos_k]{k} \leq& \left(\bVfunc[\pos_{\overline{\controle}}] - 1\right)e^{-\int_0^{\Delta t} \intensityModel[\pos(t)] \|\overline{\controle}(t) \| dt} - \left(\ibVfunc[\pos_{k+1}]{k+1} - 1\right)e^{-g_k(\pos_k,\overline{\controle})}. \nonumber
\end{align}
Some algebraic manipulations lead to:
\begin{align}
    \bVfunc[\pos(t)] - \ibVfunc[\pos_k]{k} &\leq \left(\bVfunc[\pos_{\overline{\controle}}] - 1\right)\left(e^{-\int_0^{\Delta t} \intensityModel[\pos(t)] \|\controle(t) \| dt} - e^{-g_k(\pos_k,\overline{\controle})}\right) \notag \\
    &+ e^{-g_k(\pos_k,\overline{\controle})}\left(\bVfunc[\pos_{k+1}] - \ibVfunc[\pos_{k+1}]{k+1}\right) \notag \\
    &+ e^{-g_k(\pos_k,\overline{\controle})} \left(\bVfunc[\pos_{\overline{\controle}}] - \bVfunc[\pos_{k+1}]\right). \nonumber
\end{align}
From the dynamics considered in \eqref{eq:dyn_sys}, $\pos_{\overline{\controle}} = \pos_{k+1}$, and by making use of the Mean Value Theorem, $g_k(\pos_k,\controle_k) > \overline{g}\Delta t$, and the fact that \cite[Corollary 1.4]{Cruz2002}, for a Lipschitz function $\ell(t)$ we have that the error of a trapezoidal integration is bounded by $\frac{\Delta t^2}{8}\left(\sup \dot{\ell} - \inf \dot{\ell}\right)$. It follows that:
\begin{align}
    \bVfunc[\pos(t)] - \ibVfunc[\pos_k]{k} \leq C_1 \Delta t^2 + e^{-\overline{g}\Delta t} \|\overline{V} - \overline{V}_k\|_\infty,
    \label{eq:Vtime_1}
\end{align}
for some constant $C_1$. If we consider that the $\inf$ operator is attained by $\controle^\ast$ in \eqref{eq:dpp_kruzkov}, and that:
\begin{align}
    \hat{\controle}_k = \frac{1}{\Delta t} \int_0^{\Delta t} \controle^\ast(\tau) d\tau \nonumber
\end{align}
is the control obtained by the mean of the optimal control over a time step, it follows that
\begin{align}
    \ibVfunc[\pos_k]{k} - \bVfunc[\pos(t)] \leq& \left(\ibVfunc[\pos_{k+1}]{k+1} - 1\right)e^{-g_k(\pos_k,\hat{\controle}_k)} \notag \\
    &-\left(\overline{V}(\pos_{\controle^\ast}) - 1\right)e^{-\int_0^{\Delta t} \intensityModel[\pos(t)] \|\controle^\ast(t) \| dt}. \nonumber
\end{align}
Since, for the particular case of the dynamics considered in \eqref{eq:dyn_sys}, $\pos_{\controle^\ast} = \pos(t) + \hat{\controle}_k \Delta t$, following similar arguments to the ones used in obtaining \eqref{eq:Vtime_1}, we have:
\begin{align}
    \ibVfunc[\pos_k]{k} - \bVfunc[\pos(t)] \leq C_2 \Delta t^2 + e^{-\overline{g}\Delta t} \|\overline{V} - \overline{V}_k\|_\infty, \nonumber
\end{align}
with $C_2$ some constant term. This implies that:
\begin{align}
    (1 - e^{-\overline{g}\Delta t})\|\overline{V} - \overline{V}_k\|_\infty &\leq C \Delta t^2, \nonumber \\
    \|\overline{V} - \overline{V}_k\|_\infty &\leq C \Delta t,
    \label{eq:error_time}
\end{align}
with $C = \max(C_1,C_2)$.

For the space discretization error, we analyze the errors of the value function on the grid points, by comparing \eqref{eq:dpp_discrete} and \eqref{eq:dpp_SL}. To differentiate them, we will denote the value on the grid points of \eqref{eq:dpp_SL} by $\overline{v}_k$. If we consider that the $\inf$ is attained by $\overline{\controle}$ in \eqref{eq:dpp_discrete}, it follows that:
\begin{align}
    \overline{v}_k(\pos_k) \! - \! \ibVfunc[\pos_k]{k} \leq& \left(\mathcal{I}_{\overline{v}_{k+1}}[\pos_{k+1}] \! - \! \ibVfunc[\pos_{k+1}]{k+1}\right)e^{-g_k(\pos_k,\overline{\controle})}, \nonumber
\end{align}
which, by considering that $g_k(\pos_k,\overline{\controle}) > \overline{g}\Delta t$, leads to:
\begin{align}
    \overline{v}_k(\pos_k) - \ibVfunc[\pos_k]{k} &\leq e^{-\overline{g}\Delta t}\left|\mathcal{I}_{\overline{v}_{k+1}}[\pos_{k+1}] - \mathcal{I}_{\overline{V}_{k+1}}[\pos_{k+1}]\right| \notag \\
    &+ e^{-\overline{g}\Delta t}\left|\mathcal{I}_{\overline{V}_{k+1}}[\pos_{k+1}] - \ibVfunc[\pos_{k+1}]{k+1}\right|. 
    \label{eq:err_v_space}
\end{align}
Since $\overline{V}$ is Lipschitz, then:
\begin{align}
    \left|\mathcal{I}_{\overline{V}_{k+1}}[\pos_{k+1}] - \ibVfunc[\pos_{k+1}]{k+1}\right| \leq C_3 \Delta \pos, \nonumber
\end{align}
with $C_3$ being a constant and $\Delta \pos$ the largest distance between any point and a grid point. Combining this inequality with \eqref{eq:err_v_space}, we have that:
\begin{align}
    \overline{v}_k(\pos_k) - \ibVfunc[\pos_k]{k} \leq& e^{-\overline{g}\Delta t} \|\overline{v} - \overline{V}_k\|_\infty + e^{-\overline{g}\Delta t} C_3 \Delta \pos. \nonumber
\end{align}
Since a similar bound can be found for $\ibVfunc[\pos_k]{k} - \overline{v}_k(\pos_k)$, then:
\begin{align}
    \|\overline{v} - \overline{V}_k\|_\infty \leq& e^{-\overline{g}\Delta t} \|\overline{v} - \overline{V}_k\|_\infty + C_4 \Delta \pos, \nonumber\\
    \|\overline{v} - \overline{V}_k\|_\infty \leq& C_4 \frac{\Delta \pos}{\Delta t} \label{eq:error_space}
\end{align}
with $C_4$ a constant. Combining \eqref{eq:error_time} and \eqref{eq:error_space}, we have that:
\begin{align}
    \|\overline{V} - \overline{v}\|_\infty \leq C_5 \Delta t + C_6 \frac{\Delta \pos}{\Delta t}, 
    \label{eq:consistency}
\end{align}
with $C_5$ and $C_6$ being constants. As suggested in \cite{Falcone2013}, the best coupling between $\Delta t$ and $\Delta \pos$, in this case, is given by $\Delta \pos = \Delta t^2$, indicating that our grid resolution should be finer than our time discretization resolution.

Since we have shown that our scheme is monotone, a contraction mapping, and consistent, from the Barles-Souganidis Theorem, we have proven its convergence.
\end{proof}

\begin{remark}
    Although our consistency analysis considered the specific case of the system dynamics given by \eqref{eq:dyn_sys}, similar bounds can be found for other dynamics (as long as $\dynaFunc[\pos,\controle]$ is Lipschitz), by considering the numerical integration error of solving the ordinary differential equation by the method being used).
\end{remark}

\begin{remark}
    Note that our convergence results demand that $\ell(\pos,\controle)$ in \eqref{eq:ell} is Lipschitz continuous, which demands that the Boolean disk coverage model \eqref{eq:boolean_disk_coverage_model} be approximated by a Lipschitz continuous function and that a maximum value be allowed for the Attenuated disk coverage model \eqref{eq:attenuated_disk_coverage_model} and for the Probability coverage model with noise in \eqref{eq:probability_coverage_model_noise}. In addition to this, the fact $\ell(\pos,\controle) > 0, \; \forall \pos,\controle$ requires that the intensity of the sensor field does not vanish at any point, which is essential to ensure that the problem admits a unique optimal solution.
\end{remark}

\subsection{Implementation details}


As previously discussed, the approach is composed of a time and a space discretization, and our results converge to the optimal solution as these discretization errors decrease. 
Concerning the time discretization, if the time-step $\Delta t$ is too small, it could slow down the convergence of the iterations when \emph{value iteration} is used. This problem is somewhat mitigated when employing the \emph{policy iteration}, but a suitable time-step must still be chosen. For the experiments in the next section, we have set $\Delta t = \n{0.1}$.

Although the proposed approach could be employed with a regular structured grid on the space, using rectangles as finite elements for the interpolation, unstructured grids allow for a better representation of the nonlinear profile of the node sensing functions and the transformed value function in \eqref{eq:value_kruzkov}. It also allows representing obstacles with arbitrary geometries. 
In that regard, we have used a simple heuristic to sample grid points, concentrating them near the nodes and obstacle boundaries, as shown in Fig.~\ref{fig:grid_points}. Basically, the higher the exposure value at $\pos$, the higher is the chance of it receiving a grid point. For each sensor node, we have used about \n{100} grid points.

\begin{figure}[t]
    \centering
    \includegraphics[width=\onefig\linewidth]{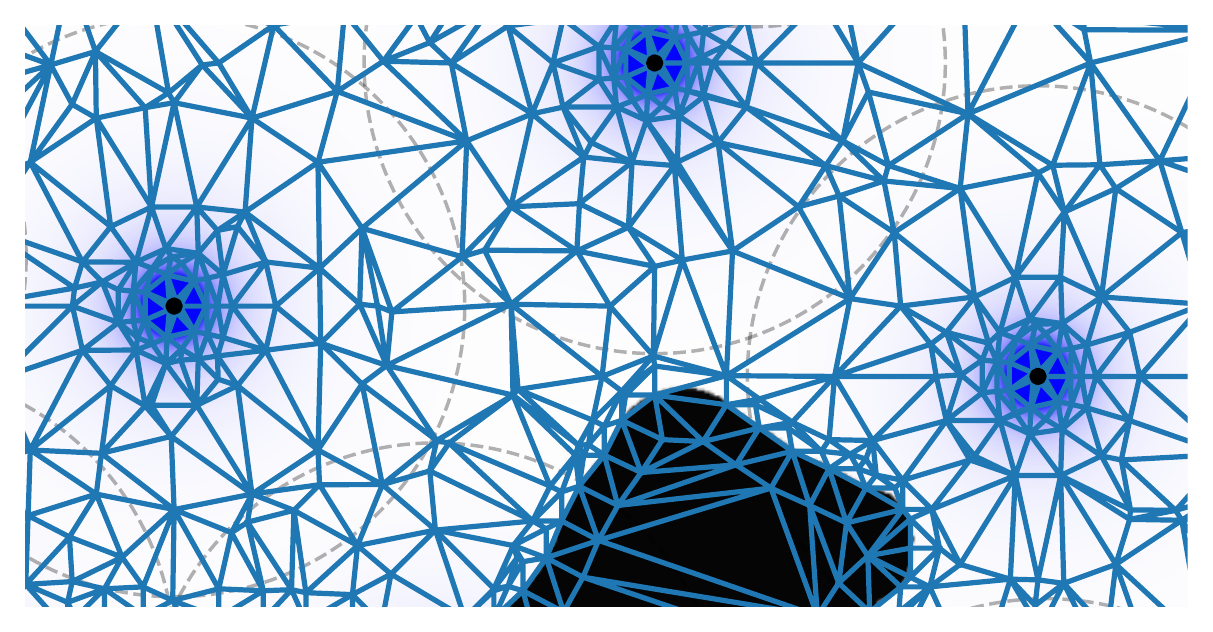}
    \caption{Spatial grid discretization: points are more concentrated around nodes (black dots) and in the obstacles boundaries (black regions).}
    \label{fig:grid_points}
    \vspace{-4mm}
\end{figure}

To solve the minimization problem in \eqref{eq:opt_policy}, we perform an exhaustive search within a discrete set of allowable velocities. Although this is not the most efficient approach, it avoids local minima.

If values of \eqref{eq:value_original} are too high, usually caused by having some sensing node too close to the goal, 
there may be numerical problems working with the Kruzkov transformed value function in \eqref{eq:value_kruzkov}, as the values might get too close to 1. To mitigate this problem, we might consider the rescaling of the intensity functions \eqref{eq:intensity_allsensors} and \eqref{eq:intensity_maximumsensor}, by dividing it by $\omega$, leading to:
\begin{align}
    \overline{I}(\pos) = \frac{\intensityModel[\pos]}{\omega}, && \omega \geq 1,
\end{align}
employed in \eqref{eq:gk} instead of the original intensity function. For the experiments in the next section, we have set $\omega = \n{100}$.

Finally, even though our conditions converge to the optimal solution, in practice, we have a sub-optimal path, and a local optimization procedure is performed over the path found by the policy in \eqref{eq:opt_policy}. This procedure checks along the path, described by a series of points in discrete-time, if each point could be substituted by a better point close to it, and replaces it if that is the case.



\section{Experiments}
\label{sec:experiments}

In this section, we first show an illustrative example to discuss different aspects of our approach. Next, we compare our results with the state-of-the-art literature considering recently available benchmarks. Finally, we provide results with cluttered environments and heterogeneous sensor nodes.

\subsection{Illustrative example}
\label{subsec:illustrative}

We begin by employing our method to the scenario presented in \cite[Section VII.A]{Ye2016Hybrid}. It consists of a network topology with 32 nodes, modeled as attenuated disk coverage functions \eqref{eq:attenuated_disk_coverage_model} with $\lambda = 4$ and $\mu = 2$. They have also adopted the maximum-sensor intensity function \eqref{eq:intensity_maximumsensor}. Considering a \n[m]{10} x \n[m]{10} space, destination point was set to $\pgoal = [10, 6.5]$, while source position was initially set to $\pinit = [0, 4]$. Subsequently, we have also determined paths for $\pinit = [0, 8]$ and $\pinit = [5, 0]$.
Fig.~\ref{fig:illustrative_example} presents the value function (gray-map background) and the \ac{MEPs} computed by our approach when applied to the aforementioned scenario, concerning all three source positions. The exposure values and execution times are shown at Tab.~\ref{tab:illustrative_results}.
\begin{figure}[!t]
    \centering
    \includegraphics[width=\onefig\linewidth]{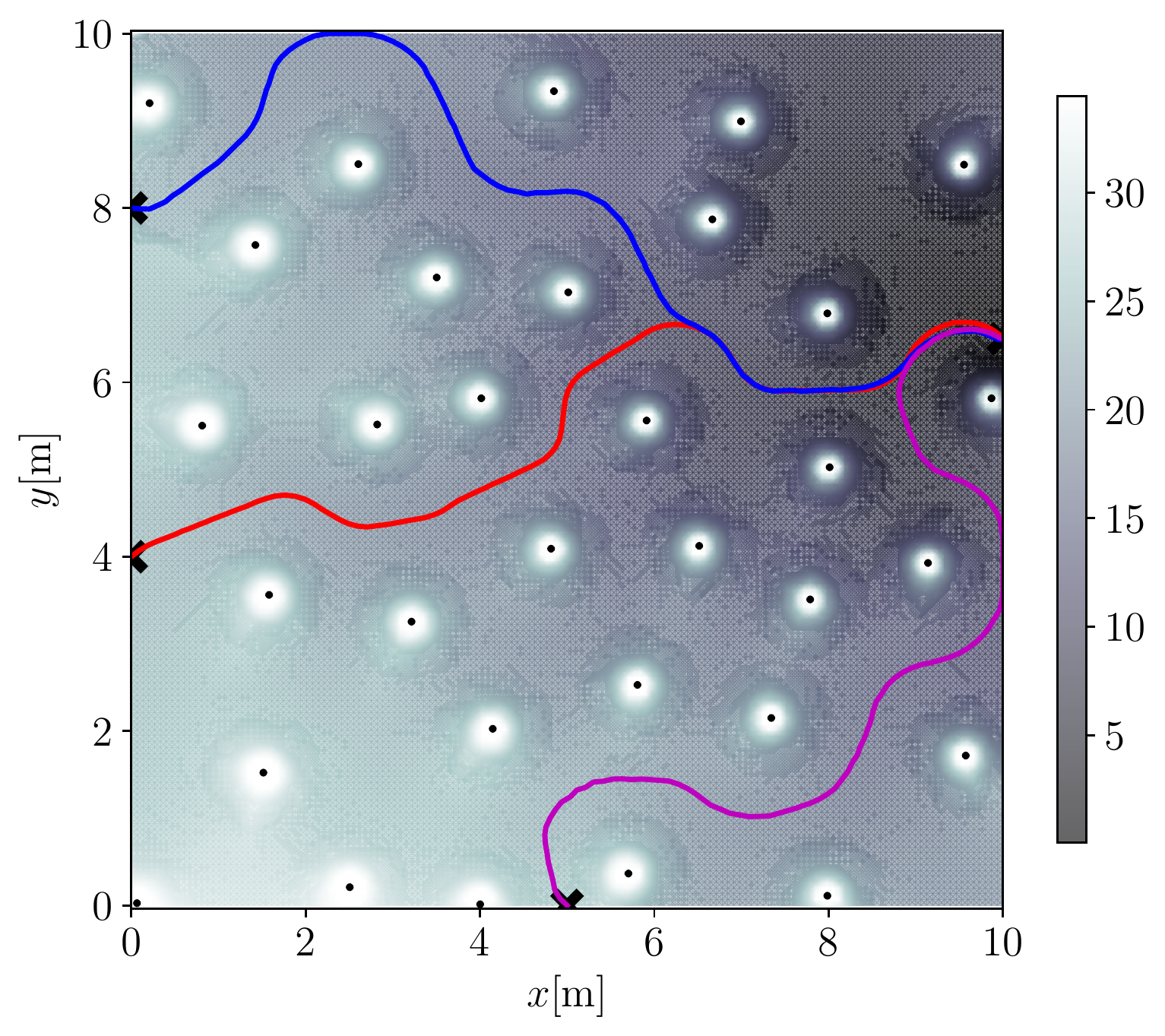}
    \caption{Example scenario (presented in \cite{Ye2016Hybrid}). The black dots are the nodes locations, while the color-map represents $\Vfunc[\pos]$. We computed three different \protect\ac{MEP}s to $\pgoal = [10, 6.5]$, the first one starting from $\pinit = [0, 4]$ (red line), and other two beginning from $\pinit = [0, 8]$ and $[5, 0]$.}
    \label{fig:illustrative_example}
\end{figure}

{
\nprounddigits{3}
\npdigits{2}{3}
\begin{table}[ht]
    \centering
    \caption{Exposure value and execution time for the paths shown in illustrative example of Fig.\ref{fig:illustrative_example}.}
    \begin{tabular}{c|cc}
        \hline
        $\pinit$ & $\boldsymbol{\exposure}$ & \bf Time(s) \\ 
        \hline
        $[0, 4]$ & \n{43.53258} & 1349 \\
        $[0, 8]$ & \n{43.12626} & 138 \\
        $[5, 0]$ & \n{46.75905} & 130 \\
         \hline
    \end{tabular}
    \label{tab:illustrative_results}
    \vspace{-3mm}
\end{table}
}

It is possible to notice that, the execution time for the first trial was \n[s]{1349}, while for the other two sources it was approximately ten times faster. The first path is slower because, in the first run, our method has to compute the value function \eqref{eq:value_original}, which is more computationally costly. However, for all subsequent computations of a \ac{MEP} (for the same $\pgoal$), we can use the previous $\Vfunc[\pos]$ illustrated at Fig.\ref{fig:illustrative_example}.
This is an essential advantage of our technique when compared to the literature since other approaches usually require a complete algorithm reset (e.g., evolutionary algorithms) for every new $\pinit$.


Tab.~\ref{tab:comparison_hga_nfe} presents a comparison between our method and the other three discussed in \cite{Ye2016Hybrid}. 
The parameter $m$ represents the grid resolution used in the algorithm. Small grid cell numbers have less exposure, especially due to the low accuracy of the results. On the other hand, high accuracy grids (with larger $m$) achieves values closer to the true exposure, however, by exponentially increasing time and storage costs \cite{Ye2016Hybrid}.
Our approach is capable of providing the most accurate value of $\exposure$ (optimal), except for some small numerical inaccuracy.

{
\nprounddigits{2}
\npdigits{3}{2}
\begin{table}[ht]
    \centering
    \caption{Comparison between our method and results in Ye et al. \cite{Ye2016Hybrid}.}
    \resizebox{\linewidth}{!}{
    \tabcolsep1mm\small
    \begin{tabular}{c|cc|cc|cc||cc}
        \hline
         & 
        \multicolumn{2}{c|}{\bf HGA-NFE \cite{Ye2016Hybrid}} & 
        \multicolumn{2}{c|}{\bf Grid-based \cite{Ye2016Hybrid}} & 
        \multicolumn{2}{c||}{\bf Voronoi-based \cite{Ye2016Hybrid}} & 
        \multicolumn{2}{c}{\bf Semi-Lagrangian} \\
        \hline
        {\bf $m$} & 
        $\boldsymbol{\exposure}$ & \bf Time(s) & 
        $\boldsymbol{\exposure}$ & \bf Time(s) & 
        $\boldsymbol{\exposure}$ & \bf Time(s) & 
        $\boldsymbol{\exposure}$ & \bf Time(s)\\
        \hline
        40 & \n{42.5916} & \n{36.3427} & \n{57.0373} & \n{42.7584} & \multirow{3}{*}{\n{47.4402}} & \multirow{3}{*}{\n{4.59764}} & \multirow{3}{*}{\bf\n{43.53258}} & \multirow{3}{*}{1349} \\
        80 & \n{43.0125} & \n{83.7606} & \n{55.0518} & \n{178.2976} & & & \\
        100 & \n{43.3752} & \n{132.909} & \n{54.6562} & \n{274.2848} & & & \\ 
         \hline
    \end{tabular}
    }
    \label{tab:comparison_hga_nfe}
    \vspace{-2mm}
\end{table}
}

\subsection{Comparative analysis with the state-of-the-art}


In order to qualitatively evaluate the performance of our Semi-Lagrangian approach, first, we define the sensing function and the field intensity function according to \cite{Binh2019Efficient} (state-of-the-art). The authors have used the probability coverage model with noise for the sensor node, given by Eq.~\eqref{eq:probability_coverage_model_noise}. Assuming the noise energy as a normal distribution with zero mean, $\eta \sim \distribuicao{N}(0, \sigma^2)$, they have simplified the sensing function, such as:
\begin{equation*}
    \sensingModel[\sensorpos, \pos] \approx -\ln \left( 1 - \Qfunc[\dfrac{A - \tfrac{\lambda}{\|\sensorpos - \pos\|^{\mu}}}{\sigma}] \right),
\end{equation*}
\noindent with
\[ \Qfunc[x] = \tfrac{1}{\sqrt{2\pi}} \int_{x}^{\infty} e^{\frac{-t^2}{2}} dt,\]
\noindent where $A = \n{6.0}$ is a threshold detection, $\lambda = \n{100.0}$, $\mu = \n{1.0}$ and $\sigma = \n{1.0}$. 
Still, for the field intensity function, results in \cite{Binh2019Efficient} are based on the all-sensor intensity model \eqref{eq:intensity_allsensors}.

Next, we apply our method to the exact scenarios described in the aforementioned work, whose sensors were deployed into a \n[m]{500} x \n[m]{500} environment. The authors have used three different random distributions to place the nodes: (i) Exponential distribution method, (ii) Uniform distribution method, and (iii) Gaussian distribution method.
They have tested a total of 240 configurations, with 80 different topologies for each of those three distribution methods. The set of topologies also present networks with a number of sensors $\nnodes$ ranging from 30 to 100\footnote{Data available at: Nguyen Thi My, Binh (2020), \emph{``Data for: Efficient Approximation Approaches to Minimal Exposure Path Problem in Probabilistic Coverage Model for Wireless Sensor Networks''}, Mendeley Data, V1, doi: 10.17632/5zh6cc2xww.1}. 
Here, however, for the sake of space, we have simulated experiments for all topologies with 30 and 100 nodes to verify the behavior of our approach subjected to a low and a high $\nnodes$.

All trials receive as input the source position $\pinit = [0, 150]$ and the destination position $\pgoal = [500, 350]$, and they were  executed using Python 3.8.5 on an Intel Core$^{\textrm{TM}}$ i7-7500U CPU \n[GHz]{2.70} x 4 and \n[GB]{16} of RAM under Ubuntu 20.04.


{
\nprounddigits{4}
\npdigits{3}{4}
\begin{table}[ht]
    \centering
    \caption{Comparison with nodes placed using the Exponential distribution method \cite{Binh2019Efficient}.}
    \resizebox{\linewidth}{!}{
    \begin{tabular}{cc|cc|cc|cc}
        \hline
        \multicolumn{2}{c|}{\bf Instance} & 
        \multicolumn{2}{c|}{\bf GB-MEP \cite{Binh2019Efficient}} & 
        \multicolumn{2}{c|}{\bf GA-MEP \cite{Binh2019Efficient}} & 
        \multicolumn{2}{c}{\bf Semi-Lagrangian} \\
        \hline
        $\boldsymbol{\nnodes}$ & \bf Ord & 
        $\boldsymbol{\exposure}$ & \bf Time (s) & 
        \bf Best $\boldsymbol{\exposure}$ & \bf Avg. time(s) & 
        $\boldsymbol{\exposure}$ & \bf Time (s)\\
        \hline
        \multirow{10}{*}{30} & 1 & \n{38.3990} & {0.631} & \n{43.3873} & 331 & \bf\n{26.2863} & 557 \\
         & 2 & \n{0.0402} & {0.695} & \n{0.0401} & 394 & \bf\n{0.02813} & 504 \\
         & 3 & \n{29.8977} & {0.974} & \n{29.5743} & 551 & \bf\n{25.3489} & 517 \\
         & 4 & \n{0.4342} & {0.429} & \n{0.4342} & 247 & \bf\n{0.3288} & 431 \\
         & 5 & \n{1.8524} & {0.784} & \bf\n{1.6056} & 411 & \bf\n{1.6056} & 477 \\
         & 6 & \n{0.0481} & {0.929} & \n{0.0421} & 320 & \bf\n{0.0406} & 724 \\
         & 7 & \n{1.0559} & {1.519} & \n{1.0604} & 313 & \bf\n{0.6674} & 437 \\
         & 8 & \n{0.0027} & {0.427} & \n{0.0023} & 224 & \bf\n{0.0021} & 426 \\
         & 9 & \n{0.0114} & {2.487} & \n{0.0109} & 256 & \bf\n{0.0102} & 443 \\
         & 10 & \n{0.1734} & {0.849} & \n{0.1280} & 385 & \bf\n{0.1239} & 491 \\
         \hline
        \multirow{10}{*}{100} & 1 & \n{5.2302} & \n{1.684} & \bf\n{4.0037} & 934 & \bf\n{4.0037} & 3085 \\
         & 2 & \n{1035.9792} & {7.304} & \n{1034.8402} & 1050 & \bf\n{903.3848} & 2619 \\
         & 3 & \n{2.4163} & {1.324} & \n{2.4069} & 691 & \bf\n{2.3007} & 1959 \\
         & 4 & \n{1.2142} & {1.066} & \n{1.1727} & 807 & \bf\n{1.0458} & 2210 \\
         & 5 & \n{6.6685} & {1.306} & \n{6.2369} & 888 & \bf\n{6.1064} & 4068 \\
         & 6 & \n{7.8457} & {1.189} & \n{6.4425} & 1030 & \bf\n{5.5539} & 2882 \\
         & 7 & \n{3.8835} & {3.012} & \n{4.3673} & 1121 & \bf\n{2.8832} & 1979 \\
         & 8 & \n{61.0651} & {1.093} & \n{60.3105} & 917 & \bf\n{51.5410} & 2060 \\
         & 9 & \n{138.5203} & {1.099} & \n{120.4673} & 1348 & \bf\n{105.7741} & 2413 \\
         & 10 & \n{1.8611} & {1.794} & \n{1.4176} & 667 & \bf\n{1.3986} & 2030 \\
         \hline
    \end{tabular}
    }
    \label{tab:comparison_exponential}
\end{table}
}

{
\nprounddigits{4}
\npdigits{3}{4}
\begin{table}[ht]
    \centering
    \caption{Comparison with nodes placed using the Uniform distribution method \cite{Binh2019Efficient}.}
    \resizebox{\linewidth}{!}{
    \begin{tabular}{cc|cc|cc|cc}
        \hline
        \multicolumn{2}{c|}{\bf Instance} & 
        \multicolumn{2}{c|}{\bf GB-MEP \cite{Binh2019Efficient}} & 
        \multicolumn{2}{c|}{\bf GA-MEP \cite{Binh2019Efficient}} & 
        \multicolumn{2}{c}{\bf Semi-Lagrangian} \\
        \hline
        $\boldsymbol{\nnodes}$ & \bf Ord & 
        $\boldsymbol{\exposure}$ & \bf Time (s) & 
        \bf Best $\boldsymbol{\exposure}$ & \bf Avg. time(s) & 
        $\boldsymbol{\exposure}$ & \bf Time (s)\\
        \hline
        \multirow{10}{*}{30} & 1 & \n{0.0053} & {3.094} & \n{0.0047} & 515 & \bf\n{0.00458} & 700 \\
         & 2 & \n{0.0117} & {2.444} & \n{0.0109} & 397 & \bf\n{0.01076} & 712 \\
         & 3 & \n{0.0082} & {1.513} & \n{0.0079} & 371 & \bf\n{0.00765} & 701 \\
         & 4 & \n{0.1768} & {5.414} & \n{0.1749} & 444 & \bf\n{0.16514} & 1353 \\
         & 5 & \n{0.0019} & {0.786} & \bf\n{0.0017} & 290 & \bf\n{0.00172} & 600 \\
         & 6 & \n{0.0146} & {4.013} & \n{0.0142} & 393 & \bf\n{0.01347} & 528 \\
         & 7 & \n{0.0026} & {0.477} & \bf\n{0.0021} & 317 & \bf\n{0.00206} & 506 \\
         & 8 & \n{0.0033} & {1.535} & \n{0.0073} & 326 & \bf\n{0.00277} & 937 \\
         & 9 & \n{0.0128} & {2.504} & \n{0.0118} & 378 & \bf\n{0.01029} & 603 \\
         & 10 & \n{0.0050} & {0.733} & \n{0.0042} & 493 & \bf\n{0.00400} & 784 \\
         \hline
        \multirow{10}{*}{100} & 1 & \n{4.1743} & {9.036} & \n{5.3521} & 1160 & \bf\n{3.35556} & 3215 \\
         & 2 & \n{5.3815} & {2.811} & \n{4.2433} & 1054 & \bf\n{4.20165} & 2095 \\
         & 3 & \n{20.5244} & {9.3} & \n{20.4461} & 1110 & \bf\n{19.73583} & 1786 \\
         & 4 & \n{0.9734} & {5.185} & \n{0.8304} & 1040 & \bf\n{0.78949} & 2198 \\
         & 5 & \n{1.2987} & {3.98} & \n{1.2134} & 1066 & \bf\n{1.10248} & 2811 \\
         & 6 & \n{1.9280} & {4.682} & \n{4.8181} & 1268 & \bf\n{1.55756} & 2794 \\
         & 7 & \n{0.6936} & {3.797} & \n{1.8955} & 1191 & \bf\n{0.61978} & 3191 \\
         & 8 & \n{3.4418} & {5.836} & \n{3.1789} & 1263 & \bf\n{2.93320} & 2792 \\
         & 9 & \n{106.9770} & {9.965} & \n{106.5400} & 1272 & \bf\n{101.05899} & 2621 \\
         & 10 & \n{5.1870} & {5.716} & \n{8.0124} & 1248 & \bf\n{4.75318} & 3047\\
         \hline
    \end{tabular}
    }
    \label{tab:comparison_uniform}
\end{table}
}

{
\nprounddigits{4}
\npdigits{3}{4}
\begin{table}[ht]
    \centering
    \caption{Comparison with nodes placed using the Gaussian distribution method \cite{Binh2019Efficient}.}
    \resizebox{\linewidth}{!}{
    \begin{tabular}{cc|cc|cc|cc}
        \hline
        \multicolumn{2}{c|}{\bf Instance} & 
        \multicolumn{2}{c|}{\bf GB-MEP \cite{Binh2019Efficient}} & 
        \multicolumn{2}{c|}{\bf GA-MEP \cite{Binh2019Efficient}} & 
        \multicolumn{2}{c}{\bf Semi-Lagrangian} \\
        \hline
        $\boldsymbol{\nnodes}$ & \bf Ord & 
        $\boldsymbol{\exposure}$ & \bf Time (s) & 
        \bf Best $\boldsymbol{\exposure}$ & \bf Avg. time(s) & 
        $\boldsymbol{\exposure}$ & \bf Time (s)\\
        \hline
        \multirow{10}{*}{30} & 1 & \n{0.0008} & {0.697} & \bf\n{0.0009} & 247 & \bf\n{0.0009} & 704 \\
         & 2 & \n{0.0035} & {0.6} & \n{0.0043} & 322 & \bf\n{0.0029} & 969 \\
         & 3 & \bf\n{0.0001} & {0.414} & \bf\n{0.0001} & 230 & \bf\n{0.0001} & 640 \\
         & 4 & \n{0.0111} & {0.767} & \n{0.0302} & 371 & \bf\n{0.0090} & 844 \\
         & 5 & \n{0.0014} & {0.924} & \bf\n{0.0012} & 307 & \bf\n{0.0012} & 850 \\
         & 6 & \n{0.0035} & {1.379} & \n{0.0037} & 233 & \bf\n{0.0034} & 751 \\
         & 7 & \bf\n{0.0031} & {0.687} & \n{0.0033} & 302 & \bf\n{0.0031} & 1092 \\
         & 8 & \n{0.0040} & {0.691} & \bf\n{0.0031} & 318 & \bf\n{0.0031} & 845 \\
         & 9 & \bf\n{0.0012} & {1.075} & \n{0.0013} & 240 & \bf\n{0.0012} & 1393 \\
         & 10 & \bf\n{0.0002} & {0.47} & \bf\n{0.0002} & 255 & \bf\n{0.0002} & 615 \\
         \hline
        \multirow{10}{*}{100} & 1 & \n{0.3134} & {2.195} & \n{0.6915} & 960 & \bf\n{0.2452} & 3865 \\
         & 2 & \n{1.1851} & {5.021} & \n{1.4138} & 811 & \bf\n{1.0962} & 2965 \\
         & 3 & \bf\n{0.0009} & {1.621} & \n{0.0010} & 434 & \bf\n{0.0009} & 5343 \\
         & 4 & \n{0.7224} & {1.274} & \n{0.5442} & 643 & \bf\n{0.5029} & 3549 \\
         & 5 & \n{0.1625} & {1.577} & \n{0.4453} & 986 & \bf\n{0.1334} & 3725 \\
         & 6 & \n{0.0585} & {1.176} & \bf\n{0.0579} & 772 & \bf\n{0.0579} & 2591 \\
         & 7 & \n{2.0298} & {3.22} & \n{2.3785} & 869 & \bf\n{1.5953} & 3954 \\
         & 8 & \n{0.1247} & {1.772} & \n{0.7191} & 798 & \bf\n{0.0909} & 4074 \\
         & 9 & \n{0.1766} & {1.404} & \n{0.1515} & 553 & \bf\n{0.1511} & 2325 \\
         & 10 & \n{0.1049} & {1.3} & \n{0.2841} & 727 & \bf\n{0.0871} & 4044 \\
         \hline
    \end{tabular}
    }
    \label{tab:comparison_gaussian}
\end{table}
}

Tables \ref{tab:comparison_exponential}, \ref{tab:comparison_uniform} and \ref{tab:comparison_gaussian} compile the results for the Exponential, Uniform and Gaussian distribution methods, respectively.
When comparing the minimum exposure value of the GA-MEP and the GB-MEP with our approach, we can see that the Semi-Lagrangian method always provides results that are smaller (or at most equivalent) to those in \cite{Binh2019Efficient}.

Here it is important to say that the exposure values for the GA-MEP given in the tables are the best ones obtained after $H$ runs of their main algorithm. Since it is based on meta-heuristics, there are no guarantees that the optimal values are always reached.
Also, the execution times shown for the GA-MEP are the average values of all trials, such that the real-time spent to reached the best $\exposure$ is $H$ times larger than those presented.
Consequently, we can claim that, based on the experiments, for $H \geq 3$ the Semi-Lagrangian approach is faster than the GA-MEP in most of the scenarios.

Table \ref{tab:comparison_percentage} compiles the percentages of improvement provided by our method over the minimum values given by both algorithms of \cite{Binh2019Efficient}. 
The highest improvements have been reached under Exponential (30) and Gaussian (100). Also, the total average of the Exponential distribution is higher than for other distributions, indicating that the GA-MEP is more susceptible to highly concentrated nodes. The average improvement for all scenarios was approximately \n[\%]{10}.

{
\nprounddigits{1}
\npdigits{3}{1}
\begin{table}[ht]
    \centering
    \caption{Percentages of improvement of our technique over \cite{Binh2019Efficient}.}
    \begin{tabular}{c|cc|cc|cc}
        \hline
        \bf  & 
        \multicolumn{2}{c|}{\bf Exponential} & 
        \multicolumn{2}{c|}{\bf Uniform} & 
        \multicolumn{2}{c}{\bf Gaussian} \\
        \hline
        \bf Ord & 
        \bf 30 & \bf 100 & 
        \bf 30 & \bf 100 & 
        \bf 30 & \bf 100 \\
        \hline
        1 & \n{31.54} & \n{0.00} & \n{2.55} & \bf\n{19.61} & \n{0.00} & \n{21.75} \\
        2 & \n{29.85} & \n{12.70} & \n{1.28} & \n{0.98} & \n{18.00} & \n{7.50} \\
        3 & \n{14.29} & \n{4.41} & \n{3.16} & \n{3.47} & \n{0.00} & \n{0.00} \\
        4 & \n{24.27} & \n{10.82} & \n{5.58} & \n{4.93} & \bf\n{19.19} & \n{7.58} \\
        5 & \n{0.00} & \n{2.09} & \n{0.00} & \n{9.14} & \n{0.00} & \n{17.94} \\
        6 & \n{3.61} & \n{13.79} & \n{5.14} & \n{19.21} & \n{2.86} & \n{0.09} \\
        7 & \bf\n{36.79} & \bf\n{25.76} & \n{1.90} & \n{10.64} & \n{0.00} & \n{21.41} \\
        8 & \n{10.00} & \n{14.54} & \bf\n{16.06} & \n{7.73} & \n{1.29} & \bf\n{27.12} \\
        9 & \n{6.42} & \n{12.20} & \n{12.80} & \n{5.14} & \n{2.50} & \n{0.26} \\
        10 & \n{3.20} & \n{1.34} & \n{4.76} & \n{8.36} & \n{5.00} & \n{17.00} \\
        \hline
        \bf Avg. & \n{16.00} & \n{9.77} & \n{5.32} & \n{8.92} & \n{4.88} & \n{12.06} \\
        \bf Std. & \n{13.49} & \n{7.92} & \n{5.17} & \n{6.22} & \n{7.41} & \n{10.20}\\
         \hline
    \end{tabular}
    \label{tab:comparison_percentage}
\end{table}
}

Fig.~\ref{fig:comparison_Binh2019Efficient} present some comparative tests compiled in the previous tables. Here, for each distribution (Exponential, Uniform, and Gaussian) and each $\nnodes = \{30, 100\}$, we have chosen the instance with the best percentage improvement reached by our method. Blue and green lines represent the \ac{MEPs} obtained by the GA-MEP and the GB-MEP, respectively, while the red line represents our Semi-Lagrangian approach.
Figures \ref{subfig:comp_30_exponential}-\ref{subfig:comp_30_gaussian} illustrate cases with few nodes, and figures \ref{subfig:comp_100_exponential}-\ref{subfig:comp_100_gaussian} are cases with a large number of sensors.
Here, it is possible to see that most of the cases with greater improvement are those where the GA-MEP fails to approximate the optimal solution. And even when it doesn't happen (Fig~\ref{subfig:comp_30_exponential}), there is a significant difference between our result and those provide by \cite{Binh2019Efficient}.

\begin{figure}[!t]
    \centering
    \subfigure[30(7) nodes, Exponential: \n{36.8}\%.]{
        \includegraphics[width=\threefig\linewidth]{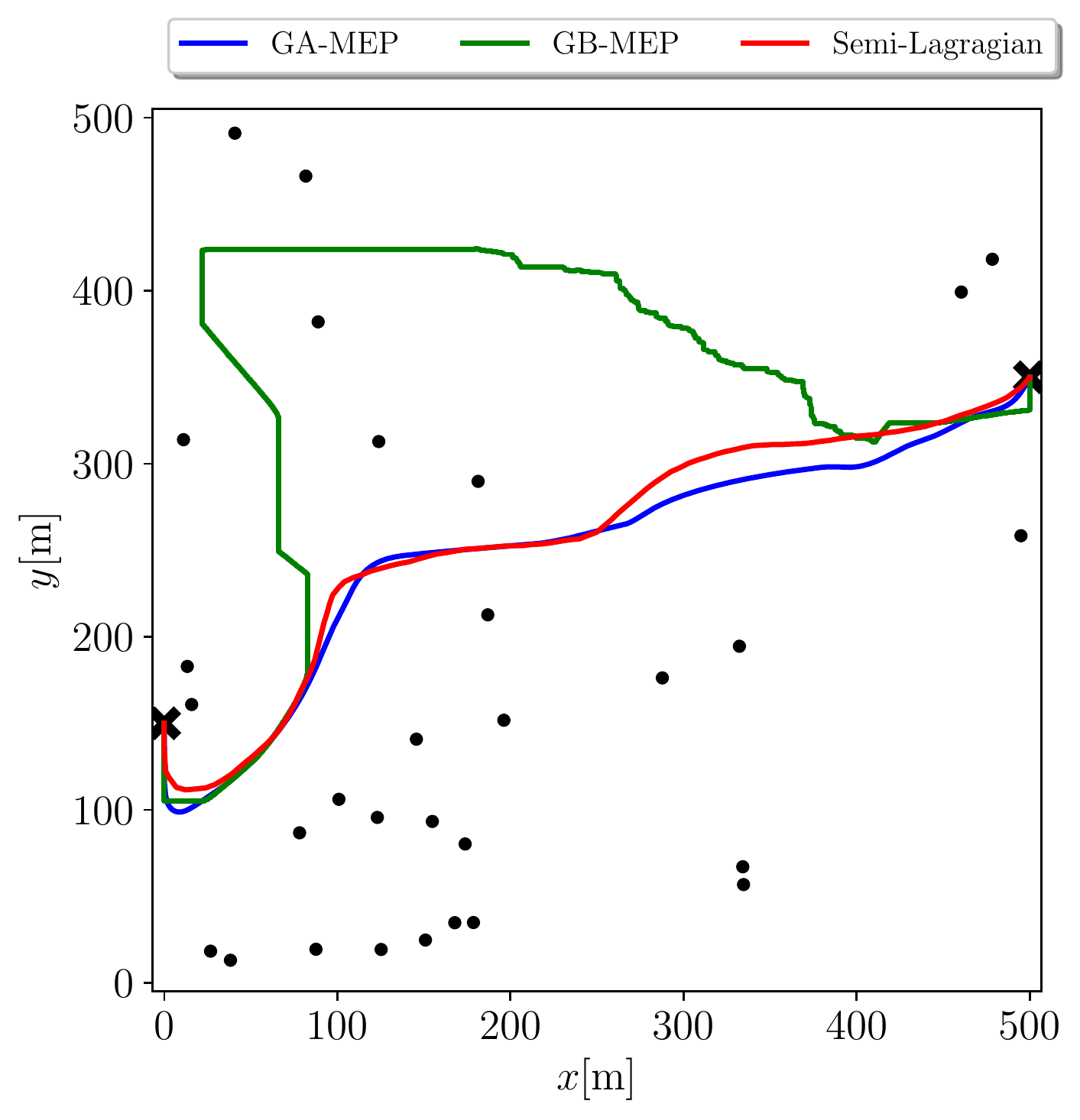}
        \label{subfig:comp_30_exponential}
    }
    \subfigure[30(8) nodes, Uniform: \n{16.1}\%.]{
        \includegraphics[width=\threefig\linewidth]{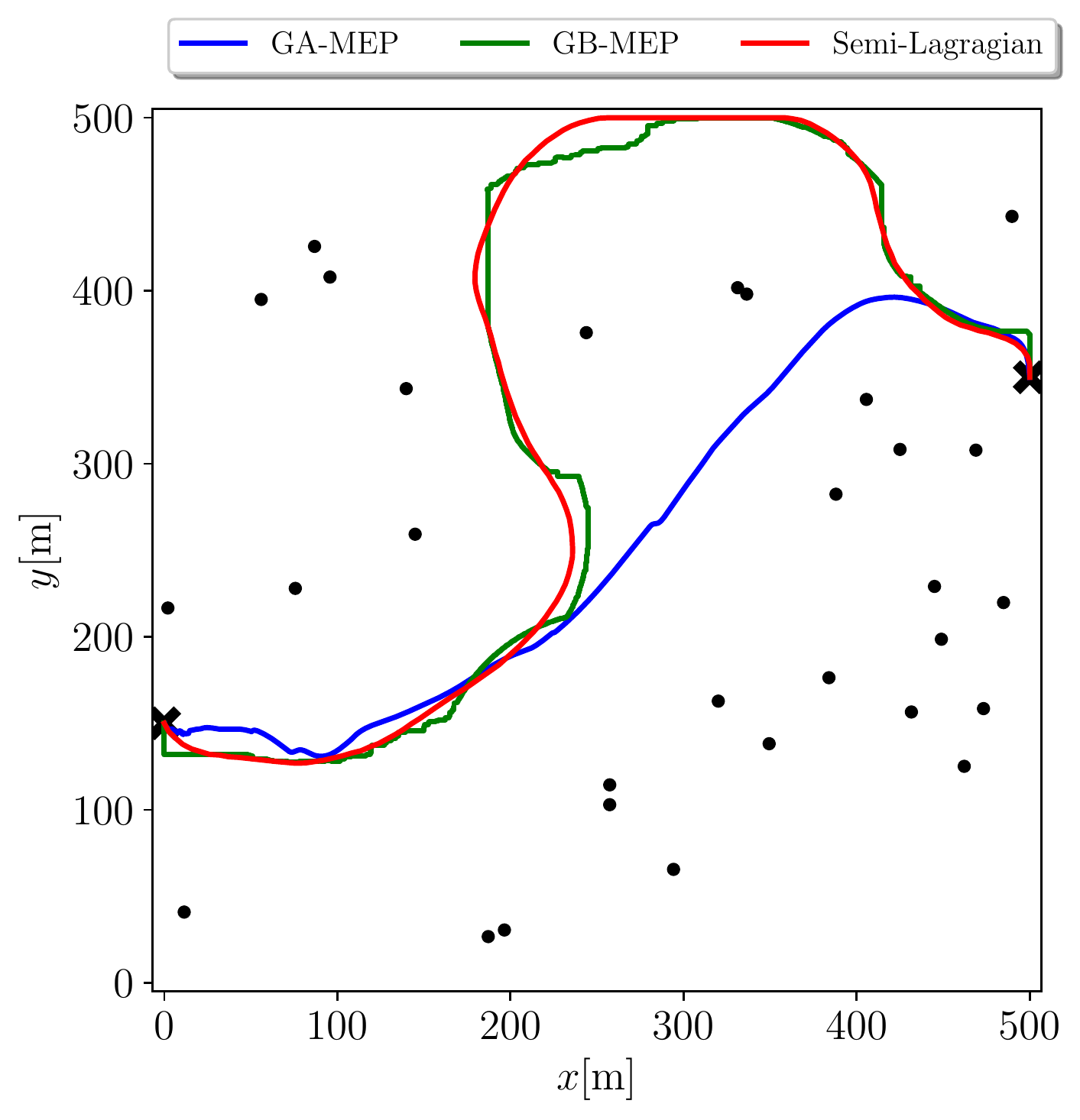}
        \label{subfig:comp_30_uniform}
    }
    \subfigure[30(4) nodes, Gaussian: \n{19.2}\%.]{
        \includegraphics[width=\threefig\linewidth]{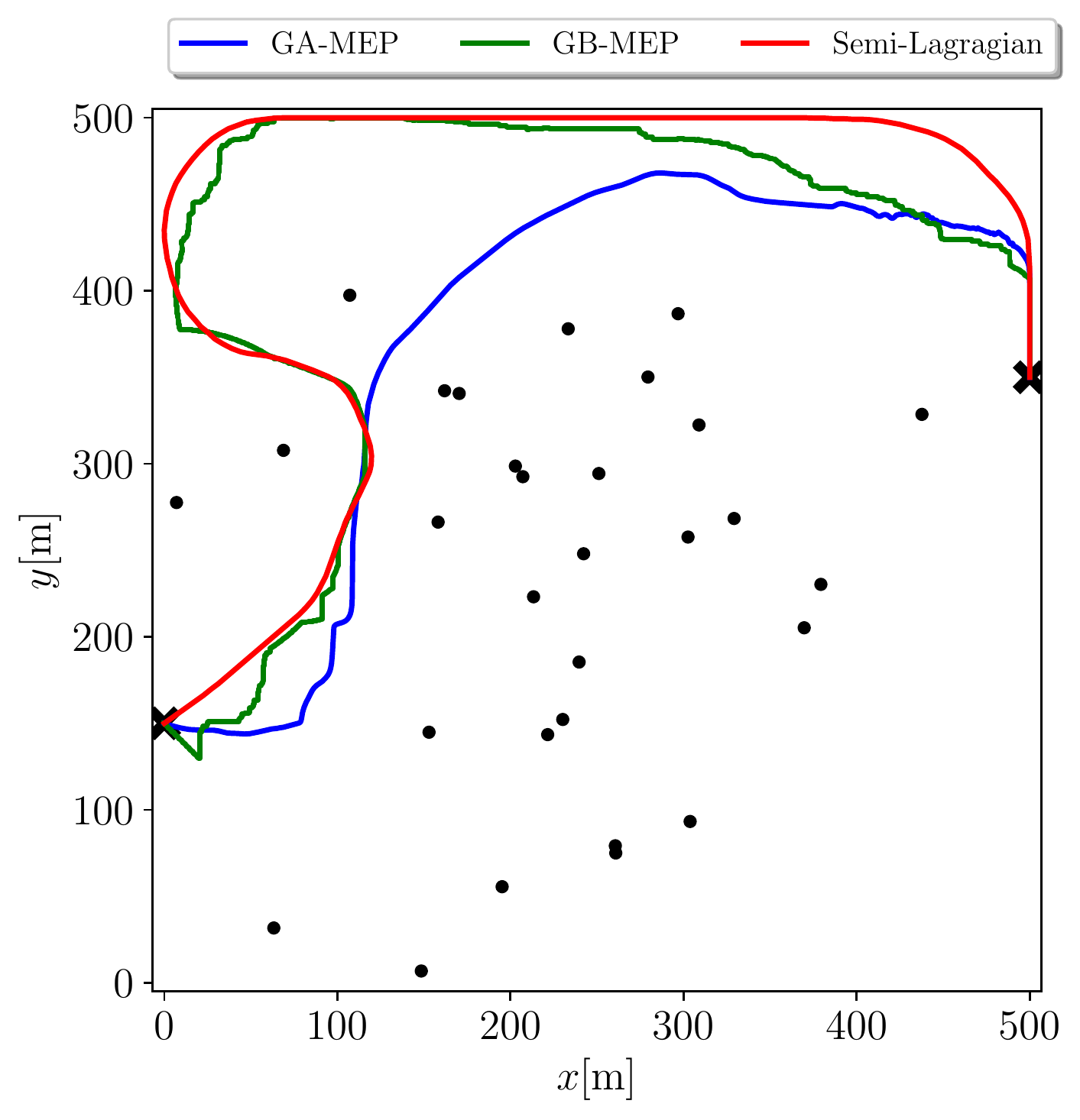}
        \label{subfig:comp_30_gaussian}
    }
    \subfigure[100(7) nodes, Exponential: \n{25.8}\%.]{
        \includegraphics[width=\threefig\linewidth]{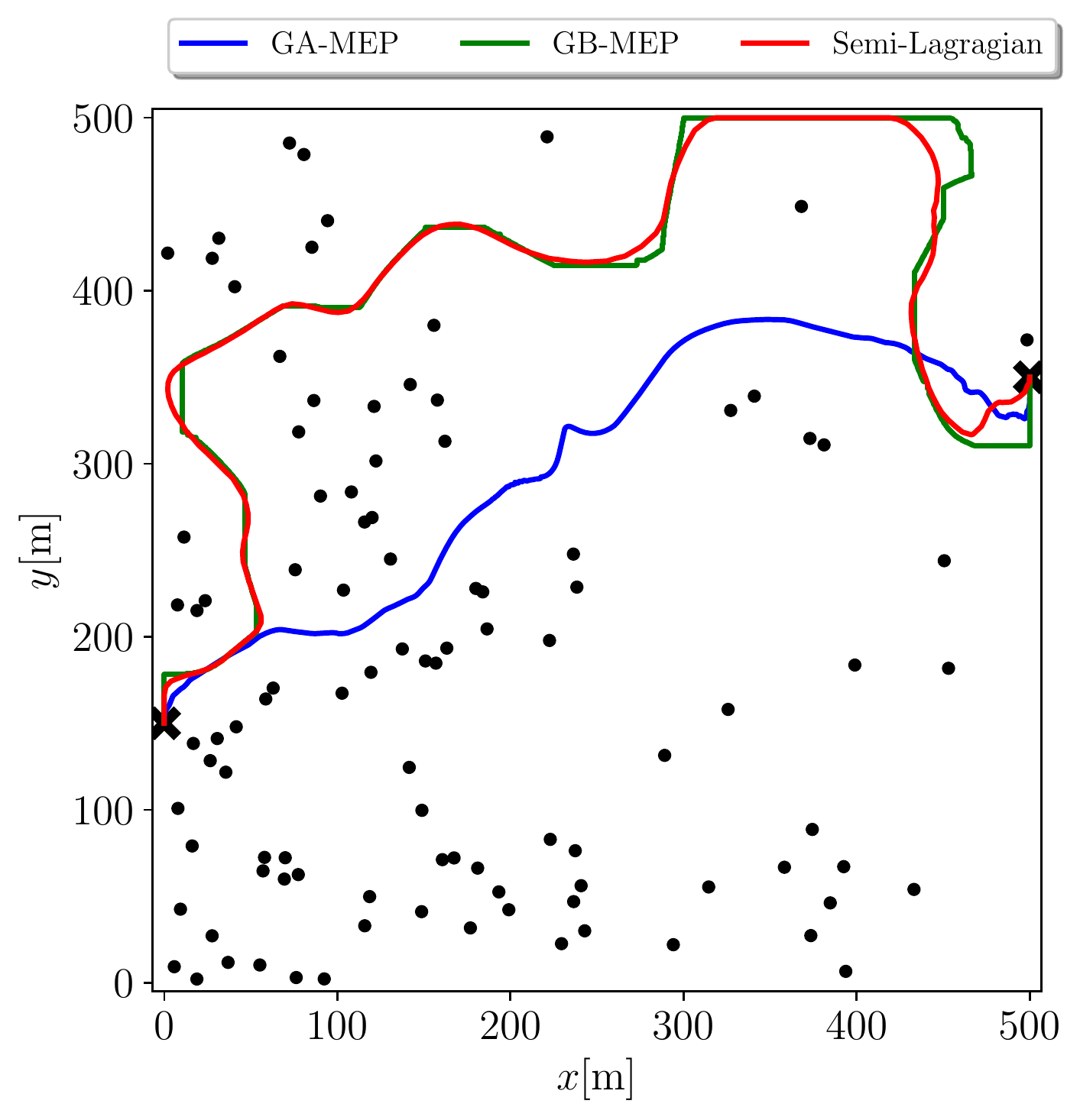}
        \label{subfig:comp_100_exponential}
    }
    \subfigure[100(1) nodes, Uniform: \n{19.6}\%.]{
        \includegraphics[width=\threefig\linewidth]{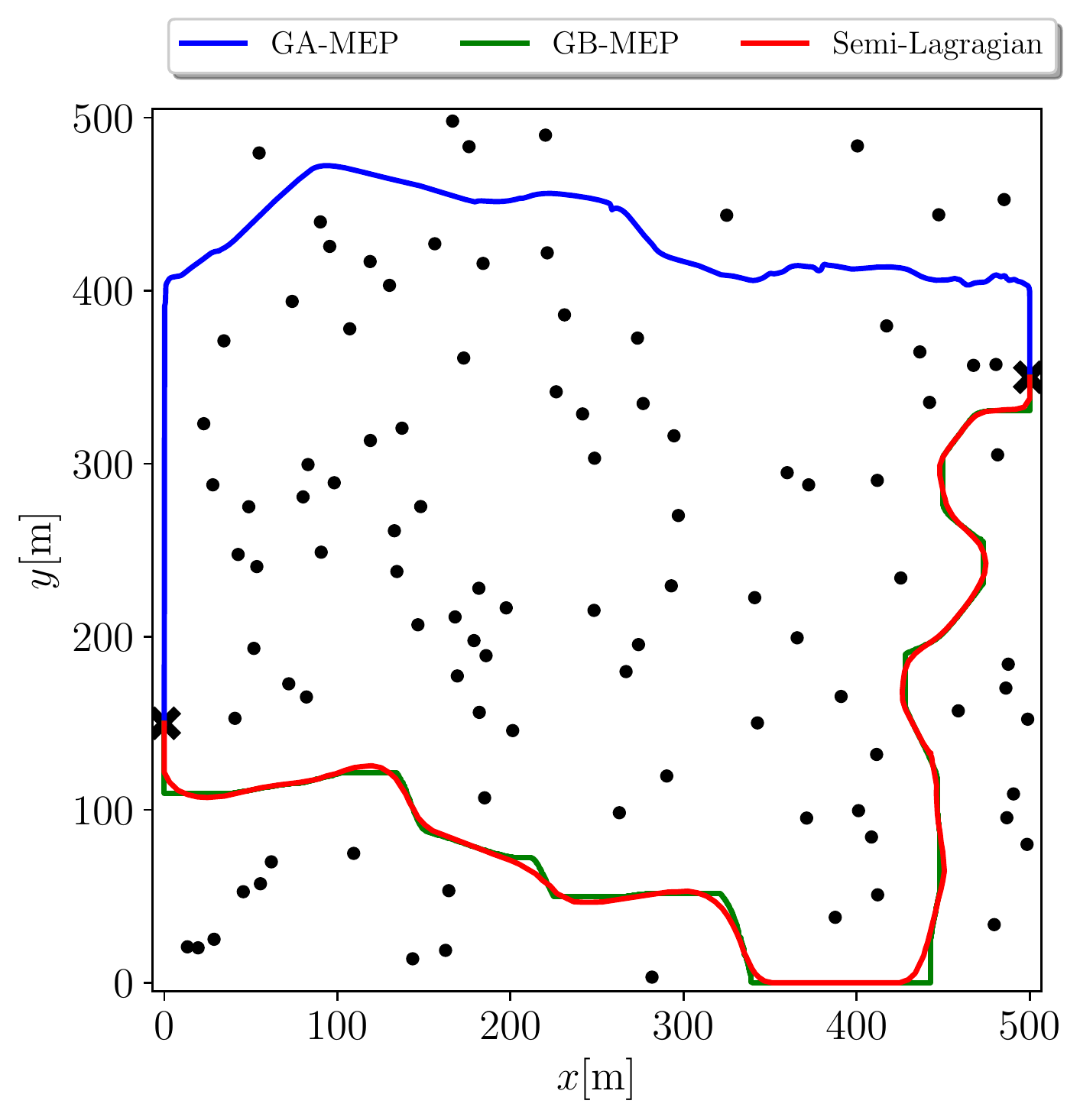}
        \label{subfig:comp_100_uniform}
    }
    \subfigure[100(8) nodes, Gaussian: \n{27.1}\%.]{ 
        \includegraphics[width=\threefig\linewidth]{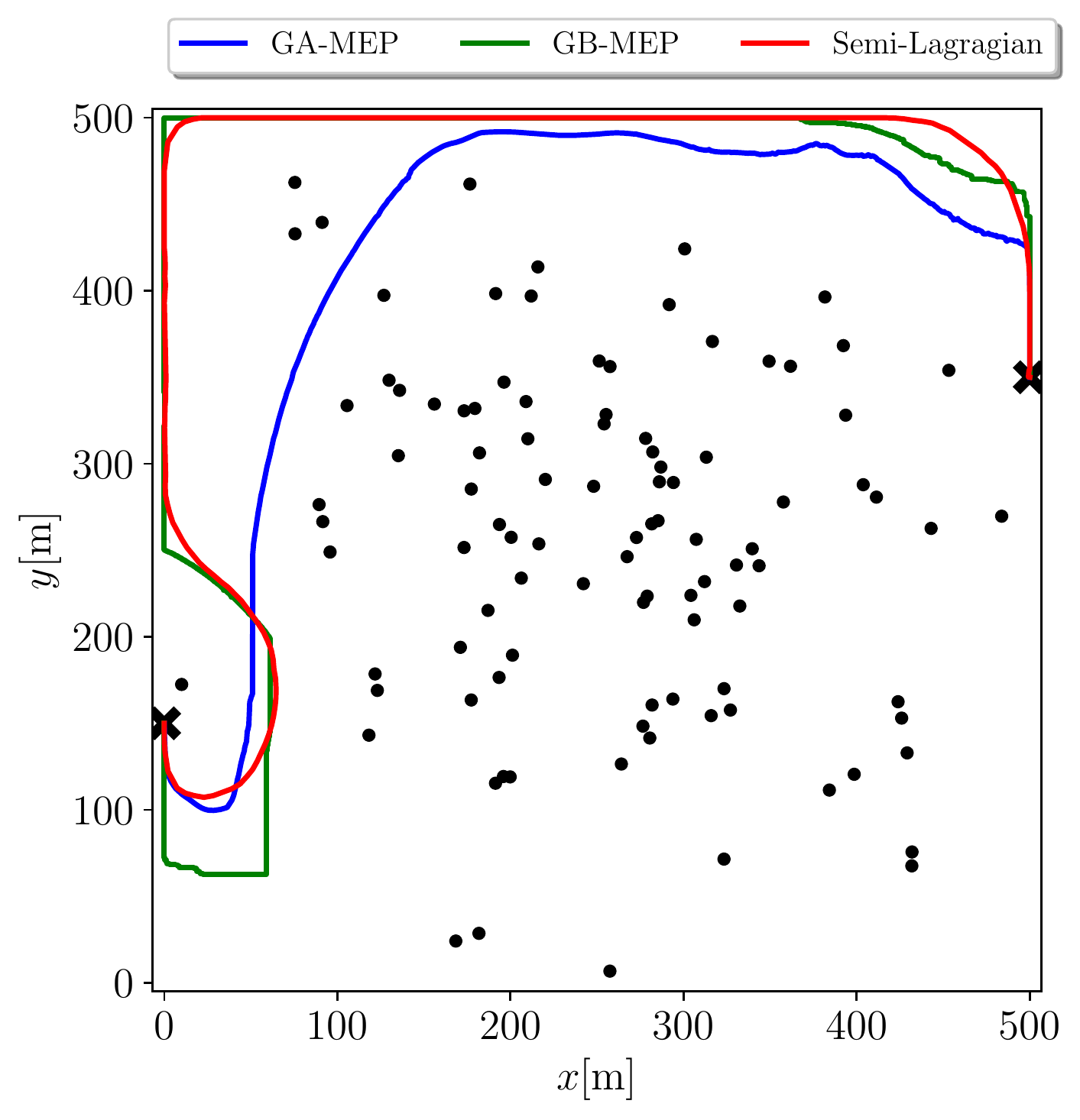}
        \label{subfig:comp_100_gaussian}
    }
    \caption{Best results of our approach in comparison with \cite{Binh2019Efficient} for each set of topologies.}
    \label{fig:comparison_Binh2019Efficient}
\end{figure}

It is also possible to notice that the highest average value was obtained in the Exponential configuration with 100 nodes. It can be explained by the concentration of sensors near the source position.

\subsection{Cluttered environments}

As previously said, the value function $\Vfunc[\cdot]$ allows the modeling of geometric constraints, given by the boundary conditions at \eqref{eq:boundary_conditions}. Therefore, it is possible to incorporate obstacles to the searching space, which makes the \ac{MEP} problem more attainable to real-world scenarios.

In Fig.~\ref{fig:obstacles}, we progressively added obstacles to the base scenario (Fig.~\ref{fig:illustrative_example}). Fig.~\ref{subfig:Ye2016_cave0.pdf} shows the original \ac{MEP} starting from $\pinit = [0, 4]$, whose exposure was previously computed as {\nprounddigits{2}\n{43.53258}} (Tab.~\ref{tab:comparison_hga_nfe}). 
When two obstacles are added (Fig.~\ref{subfig:Ye2016_cave1.pdf}), a new path is computed and $\exposure$ is increased to {\nprounddigits{2}\n{50.78452}}.
After the incorporation of five other obstacles (Fig.~\ref{subfig:Ye2016_cave2.pdf}), the path exposure reaches {\nprounddigits{2}\n{54.73732}}.
One can see that as more geometric constraints (i.e., non-navigable regions) are incorporated into the environment, the agent's exposure to the sensor network might remain the same or more commonly increase.


\begin{figure}[!t]
    \centering
    \nprounddigits{3}
    \subfigure[Empty environment: $\exposure = \n{43.53258}$.]{
        \includegraphics[width=\threefig\linewidth]{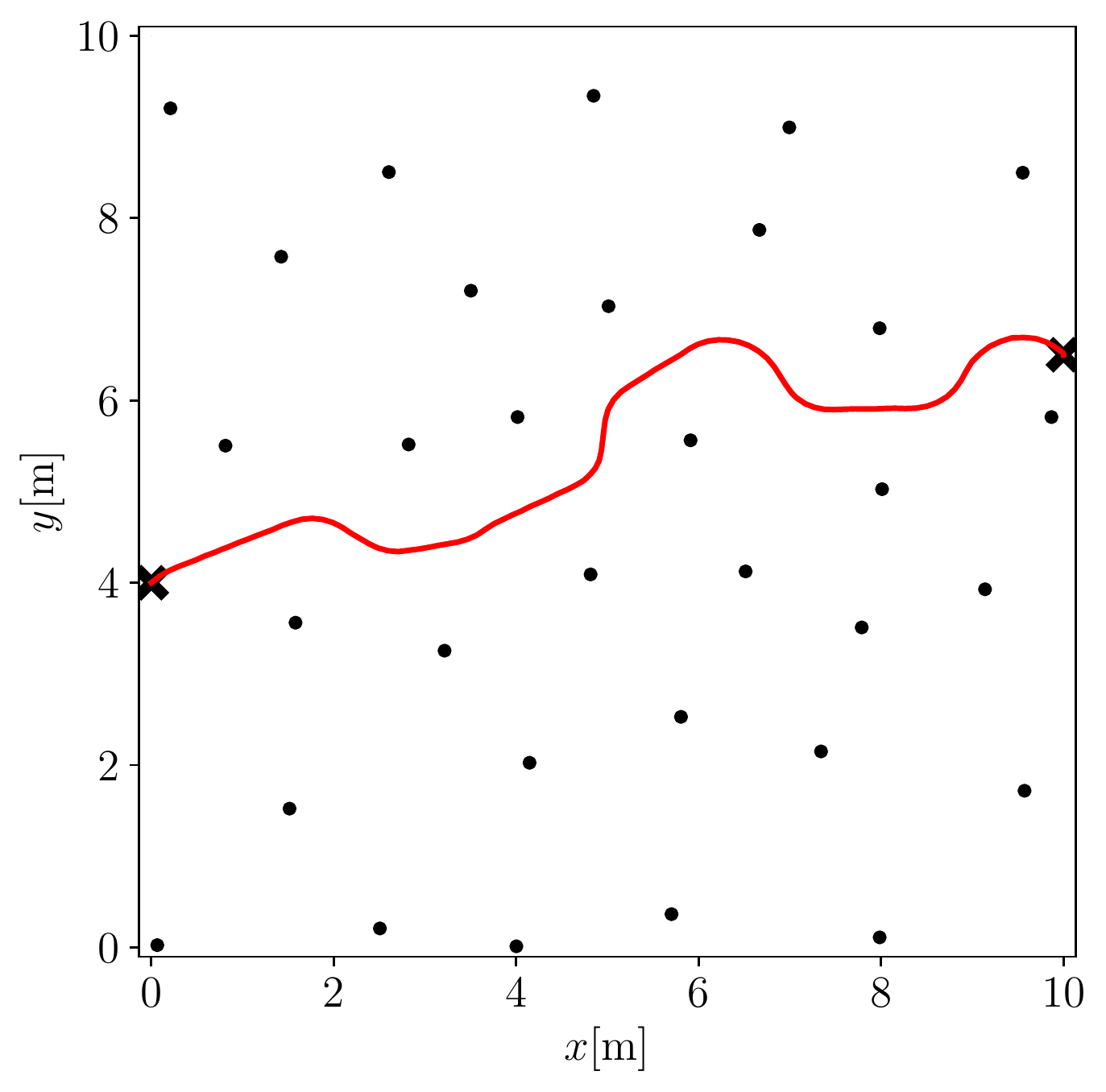}
        \label{subfig:Ye2016_cave0.pdf}
    }
    \subfigure[Two added obstacles: $\exposure = \n{50.78452}$.]{
        \includegraphics[width=\threefig\linewidth]{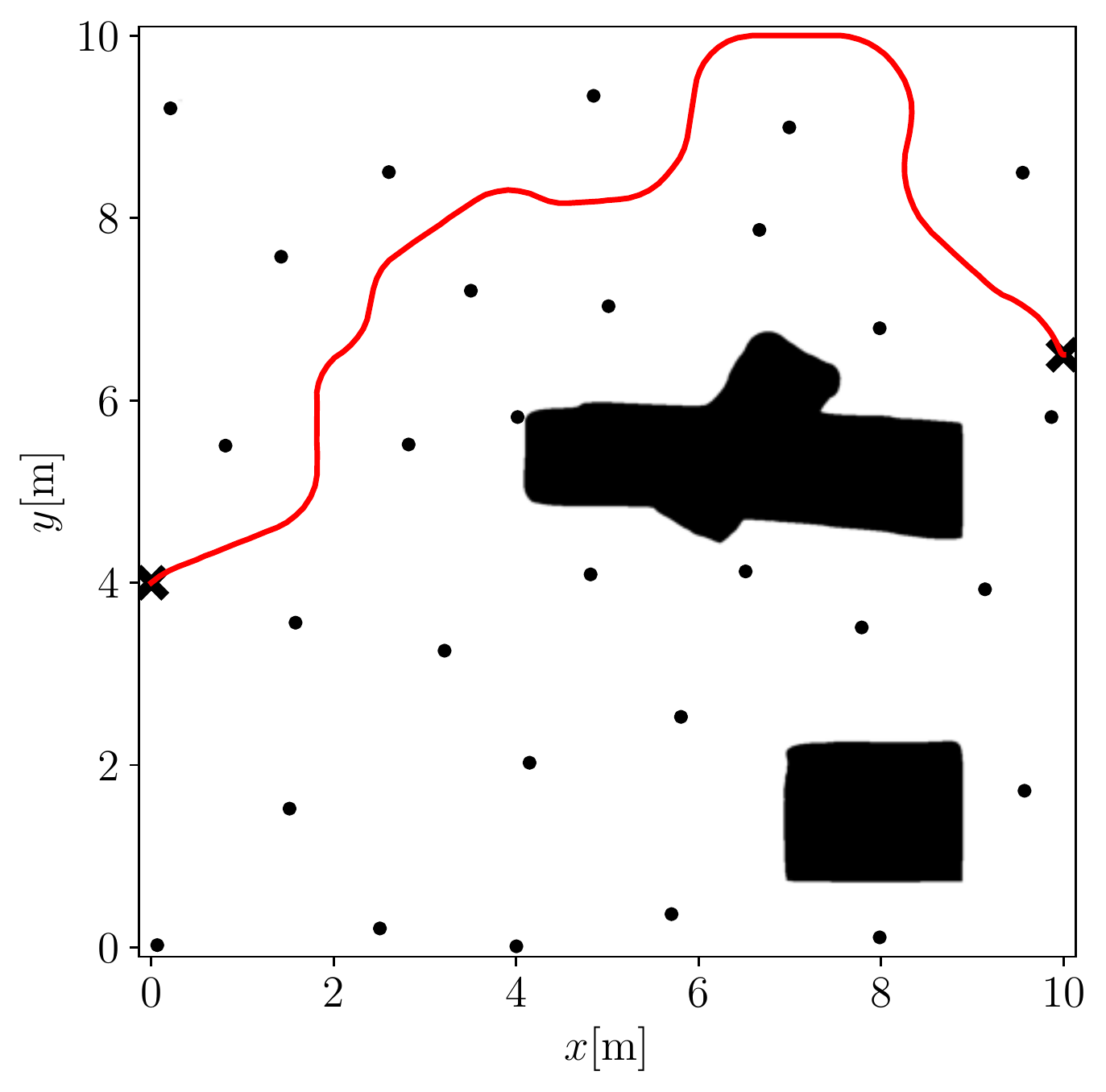}
        \label{subfig:Ye2016_cave1.pdf}
    }
    \subfigure[Five more obstacles added: $\exposure = \n{54.73732}$.]{
        \includegraphics[width=\threefig\linewidth]{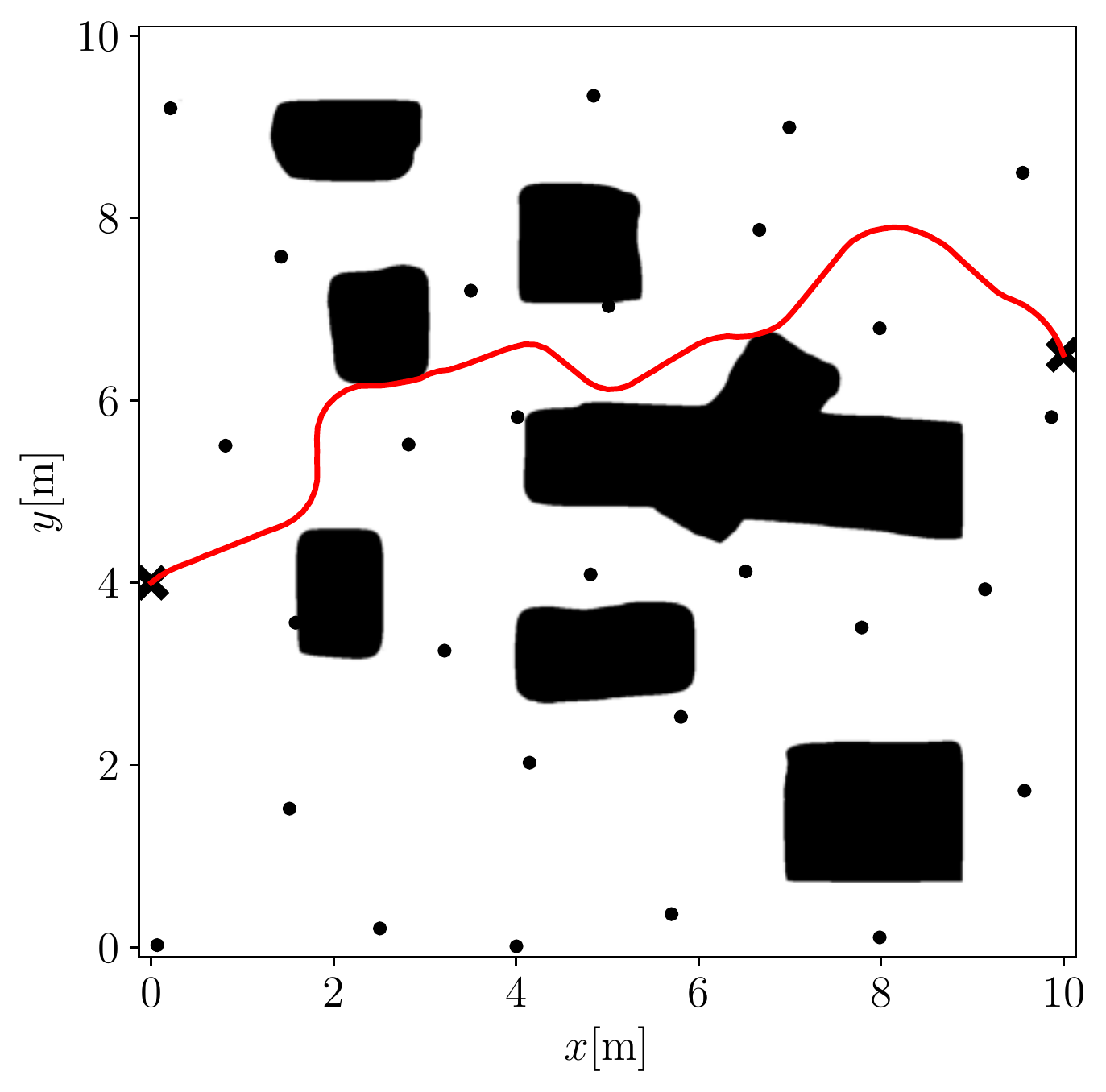}
        \label{subfig:Ye2016_cave2.pdf}
    }
    \caption{\protect\ac{MEP} computed by our method in cluttered environments: a) empty scenario discussed in Sec.~\ref{subsec:illustrative}; b) two obstacles have been added, forcing the method to modify the \protect\ac{MEP}; c) five more obstacles have been incorporated to the space, resulting in an even higher final exposure.}
    \label{fig:obstacles}
\end{figure}

\subsection{Heterogeneous networks}

Finally, we applied our method to a sensor field composed of heterogeneous nodes. We considered the same scenario described in Sec.~\ref{subsec:illustrative}, however, here we have randomly chosen the parameter $\lambda$, as 1 or 3, for the sensing model \eqref{eq:attenuated_disk_coverage_model}.


Fig.~\ref{fig:heterogeneous} presents the \ac{MEP} computed for the heterogeneous network. We highlight that the dashed circles do not represent limits to the detection range of the sensors, but they only serve to illustrate the heterogeneity of the nodes.


\begin{figure}[!t]
    \centering
    \nprounddigits{3}
    \includegraphics[width=.45\linewidth]{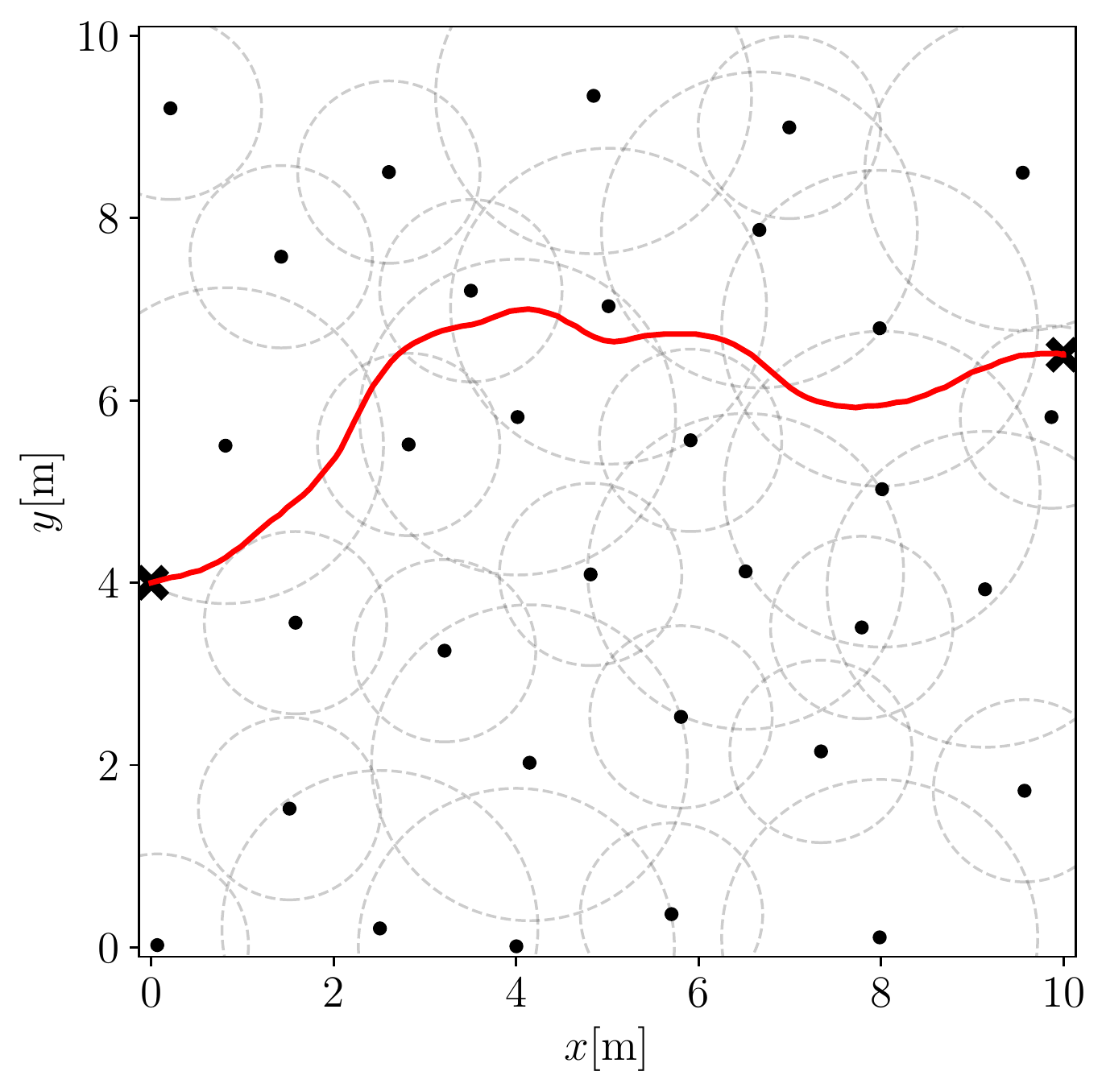}
    \caption{\protect\ac{MEP} computed for a scenario with a heterogeneous sensor network: $\exposure = \n{32.03156}$ and execution time of 1263 s.}
    \label{fig:heterogeneous}
\end{figure}

\section{Conclusion and future work}
\label{sec:conclusion}

The use of \ace{WSNs} is continuously increasing in several civilian and military applications. In this context, the \ace{MEP} between two arbitrary positions in the sensing field represents the path that minimizes the likelihood of a moving target being detected. Finding the \ac{MEP} is critical since the exposure allows estimating the coverage quality of a known deployed \ac{WSN}.

In this paper, we proposed an algorithm based on policy iteration for determining an optimal control solution. This novel approach solves the \ac{MEP} problem ensuring convergence to the optimal solution, given results that overcame the state-of-the-art in all instances, with an average improvement of about \n[\%]{10} over available benchmarks.
Our formulation allows us to incorporate the target's dynamics and tackle geometric constraints (such as obstacles), as well as many different sensing models and intensity functions for homogeneous and heterogeneous network topologies.
The method is faster than the best one in the literature, even when employed for a very large number of nodes.

As future work, we intend to extend our method to larger search spaces, where it is possible to incorporate nonholonomic constraints to the target, three-dimensional environments, and directional sensing models (such as cameras), among others. To do this, it is imperative to improve the efficiency of the algorithm, possibly resorting to meta-heuristics and numerical approximations. 
We also expect to deal with dynamic scenarios, where node sensing functions vary not only with the sensors' location but also along time.


\bibliographystyle{elsarticle-num}

\bibliography{bibliography}

{
\footnotesize
\bigskip
\noindent \textbf{Armando Alves Neto} received the B.S.E. degree in Automation and Control Engineering from the Universidade Federal de Minas Gerais in 2006, and S.M. and Ph.D. degrees in Computer Science from UFMG in 2008 and 2012, respectively. He is an Assistant Professor at the Department of Electronic Engineering at UFMG. Research interests include real-time motion planning, multi-agent control, robust control, and collision avoidance strategies.

\medskip

\noindent\textbf{Victor Costa da Silva Campos} is a Professor at the Department of Electronics Engineering at Universidade Federal de Minas Gerais (UFMG) since 2018. He received the B.Eng. degree in Control and Automation Engineering from UFMG in 2009, and the M.Sc. and PhD degrees in Electrical Engineering also from UFMG in 2011 and 2015, respectively. His research interests include robust control, adaptive control, Takagi-Sugeno fuzzy systems, Computational Intelligence applied to control systems and motion planning.

\medskip

\noindent\textbf{Douglas Guimarães Macharet} is an Assistant Professor at the Department of Computer Science (DCC) of the Universidade Federal de Minas Gerais (UFMG). He received an M.Sc. and D.Sc. degrees in Computer Science from the same university in 2009 and 2013, respectively. He is with the Computer Vision and Robotics Laboratory (VeRLab), and his main research interests are in mobile robotics, focusing on motion planning and navigation, multi-robot systems, and swarm robotics and human-robot interaction.
}
\end{document}